\documentclass{article}

\usepackage{microtype}
\usepackage{graphicx}
\usepackage{subcaption}
\usepackage{booktabs} 
\usepackage{wrapfig}
\usepackage{enumitem}
\usepackage{amsthm}

\usepackage[noend]{algorithmic}
\algsetup{indent=8pt}

\usepackage{hyperref}

\usepackage[accepted]{icml2021}

\icmltitlerunning{PACOH: Bayes-Optimal Meta-Learning with PAC-Guarantees}

\usepackage{amsmath,amsfonts,bm,amssymb}

\newtheorem{theorem}{Theorem}
\newtheorem{corollary}{Corollary}
\newtheorem{lemma}{Lemma}
\newtheorem{proposition}{Proposition}

\def\eqref#1{equation~\ref{#1}}

\def\1{\bm{1}}

\def\eps{{\epsilon}}

\DeclareMathAlphabet{\mathsfit}{\encodingdefault}{\sfdefault}{m}{sl}
\SetMathAlphabet{\mathsfit}{bold}{\encodingdefault}{\sfdefault}{bx}{n}

\newcommand{\E}{\mathbb{E}}

\newcommand{\R}{\mathbb{R}}

\DeclareMathOperator*{\argmax}{arg\,max}
\DeclareMathOperator*{\argmin}{arg\,min}

\newcommand{\rom}[1]
    {\text{\MakeUppercase{\romannumeral #1}}}

\newcommand{\bX}{\mathbf{X}}
\newcommand{\bx}{\mathbf{x}}

\newcommand{\by}{\mathbf{y}}

\newcommand{\given}{\, \vert \,}

\newcommand{\calA}{\mathcal{A}}

\newcommand{\calD}{\mathcal{D}}

\newcommand{\calH}{\mathcal{H}}

\newcommand{\calL}{\mathcal{L}}
\newcommand{\calM}{\mathcal{M}}
\newcommand{\calN}{\mathcal{N}}

\newcommand{\calP}{\mathcal{P}}
\newcommand{\calQ}{\mathcal{Q}}

\newcommand{\calT}{\mathcal{T}}
\newcommand{\calU}{\mathcal{U}}

\newcommand{\calX}{\mathcal{X}}
\newcommand{\calY}{\mathcal{Y}}
\newcommand{\calZ}{\mathcal{Z}}

\newcommand{\expect}[2]{\mathbb{E}_{#1} \; #2}

\renewcommand{\overline}{\tilde}

\newcommand{\beginsupplement}{%
        \setcounter{table}{0}
        \renewcommand{\thetable}{S\arabic{table}}%
        \setcounter{figure}{0}
        \renewcommand{\thefigure}{S\arabic{figure}}%
     }

\def\zzzforegreg#1\unskip{}

\newcommand{\vspacecaption}{\vspace{-4pt}}
\newcommand{\vspacecaptionlow}{\vspace{-3pt}}
\newcommand{\vspacesubcaption}{\vspace{-2pt}}
\newcommand{\vspaceequation}{\vspace{-2pt}}

\begin{document}

\twocolumn[
\icmltitle{PACOH: Bayes-Optimal Meta-Learning with PAC-Guarantees}




\begin{icmlauthorlist}
\icmlauthor{Jonas Rothfuss}{eth}
\icmlauthor{Vincent Fortuin}{eth}
\icmlauthor{Martin Josifoski}{epfl}
\icmlauthor{Andreas Krause}{eth}
\end{icmlauthorlist}

\icmlaffiliation{eth}{ETH Zurich, Switzerland}
\icmlaffiliation{epfl}{EPFL, Switzerland}

\icmlcorrespondingauthor{Jonas Rothfuss}{jonas.rothfuss@inf.ethz.ch}

\icmlkeywords{Machine Learning, ICML, Meta-Learning, Transfer Learning, Multi-task learning, PAC-Bayes, Learning Theory, Life-Long learning}

\vskip 0.3in
]



\printAffiliationsAndNotice{}  

\begin{abstract}

\looseness -1 
{\em Meta-learning} can successfully acquire useful inductive biases from data. Yet, its generalization properties to unseen learning tasks are poorly understood. Particularly if the number of meta-training tasks is small, this raises concerns about overfitting.
We provide a theoretical analysis using the {\em PAC-Bayesian} framework and derive novel {\em generalization bounds} for meta-learning.
Using these bounds, we develop a class of PAC-optimal meta-learning algorithms with performance guarantees and a principled {\em meta-level regularization}. Unlike previous PAC-Bayesian meta-learners, our method results in a standard stochastic optimization problem which can be solved efficiently and {\em scales well}.
When  instantiating our {\em PAC-optimal hyper-posterior (PACOH)} with Gaussian processes and Bayesian Neural Networks as base learners, the resulting methods yield state-of-the-art performance, both in terms of predictive accuracy and the quality of uncertainty estimates. Thanks to their principled treatment of uncertainty, our meta-learners can also be successfully employed for {\em sequential decision problems}.

\end{abstract}
\vspacecaption
\section{Introduction}
\vspacecaptionlow

\looseness -1 \emph{Meta-learning} aims to extract prior knowledge from data, accelerating the learning process for new learning tasks \citep{thrun1998}. 
Most existing meta-learning approaches focus on situations where the number of tasks is large \citep[e.g.,][]{finn2017model, garnelo2018neural}. In many practical settings, however, the number of tasks available for meta-training is rather small.
In those settings, there is a risk of {\em overfitting to the meta-training tasks} \citep[meta-overfitting, cf.][]{qin2018rethink}, thus impairing the performance on yet unseen target tasks.
Hence, a key challenge is how to {\em regularize} the meta-learner to ensure its  {\em generalization to unseen tasks}.

\begin{figure}
	\centering
	\includegraphics[width=1.0\linewidth]{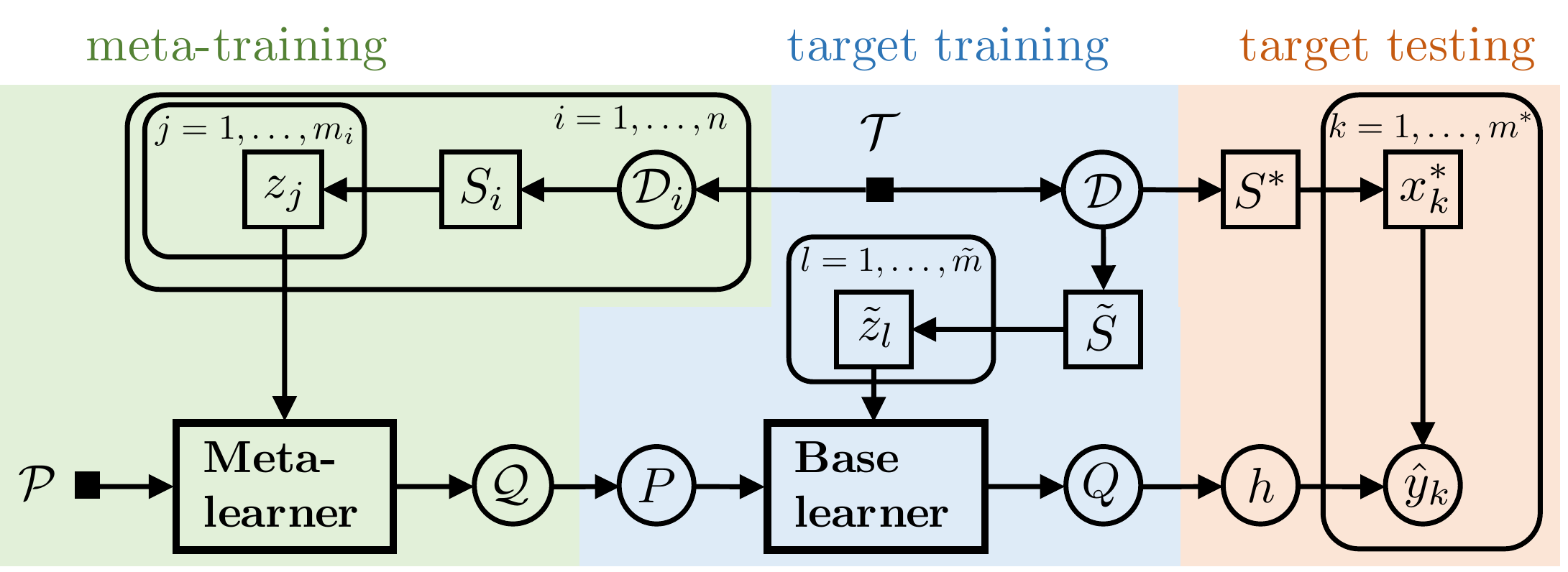} \vspace{-16pt}
	\caption{Overview of the described meta-learning framework with hyper-prior $\mathcal{P}$, hyper-posterior $\mathcal{Q}$, target prior $P$ and target posterior $Q$. The data-generating process in our setup is described by a  meta-learning
	environment $\mathcal{T}$, meta-train task distributions $\mathcal{D}_i$, target task distribution $\mathcal{D}$, data sets $S$ and data points $z = (x,y)$.
	\vspace{-12pt}}
	\label{fig:overview}
\end{figure}

\looseness -1 The PAC-Bayesian framework provides a rigorous way to reason about the generalization performance of learners \citep{mcallester1999some}.
However, initial PAC-Bayesian analyses of meta-learners \citep{pentina2014pac, amit2017meta} only consider {\em bounded} loss functions, which precludes important applications such as regression or probabilistic inference, where losses are typically unbounded.
More crucially, they rely on a challenging nested optimization problem, which is computationally much more expensive than standard meta-learning approaches.

To overcome these issues, we derive the  first {\em PAC-Bayesian bound} for meta-learners with {\em unbounded loss} functions. For Bayesian learners, we further tighten our PAC-Bayesian bounds, relating them directly to the marginal log-likelihood of the Bayesian model and thus {\em avoiding the reliance on nested optimization}. 
This allows us to derive the {\em PAC-optimal hyper-posterior (PACOH)}, which promises strong performance guarantees and a principled meta-level regularization.
Importantly, it can be approximated using standard variational methods \citep{Blei2016}, giving rise to a range of {\em scalable} meta-learning algorithms.

\looseness -1 In our experiments, we instantiate our framework with Gaussian Processes (GPs) and  Bayesian Neural Networks (BNNs) as base learners. Across several regression and classification environments, our proposed approach achieves {\em state-of-the-art} predictive {\em accuracy}, while also improving the {\em calibration} of the  uncertainty estimates. Further, we demonstrate that {\em PACOH} effectively {\em alleviates the meta-overfitting problem}, allowing us to successfully extract inductive bias from as little as five tasks while reliably reasoning about the learner's epistemic uncertainty. Thanks to these properties, {\em PACOH} can also be employed in a broad range of  {\em sequential decision problems}, which we showcase through a real-world Bayesian optimization task concerning the development of vaccines. 
\vspacecaption
\section{Related Work}
\vspacecaptionlow
\label{sec:related_work}

\textbf{Meta-learning.~} A range of methods in meta-learning attempt to learn the ``learning program" in form of a recurrent model \citep{Hochreiter2001, Andrychowicz2016, Chen2017a}, learn an embedding space shared across tasks \citep{snell2017prototypical, vinyals2016matching} or the initialization of a NN so it can be quickly adapted to new tasks \citep{finn2017model, nichol2018firstorder, rothfuss2019promp}. A group of recent methods also use probabilistic modeling to enable uncertainty quantification \citep{kim2018bayesian, finn2018probabilistic, garnelo2018neural}. Though the mentioned approaches are able to learn complex inference patterns, they rely on settings where meta-training tasks are abundant and fall short of performance guarantees. 
While the risk of meta-overfitting has previously been noted \citep{qin2018rethink, fortuin2019deep, yin2020meta}, it still lacks a rigorous formal analysis under realistic assumptions (e.g., unbounded loss functions). Addressing this issue, we study the generalization properties of meta-learners within the PAC-Bayesian framework, and, based on that, contribute a novel meta-learning approach with principled meta-level regularization.

\textbf{PAC-Bayesian theory.~} \looseness -1 Previous work presents generalization bounds for randomized predictors, assuming a prior to be given exogenously \citep{mcallester1999some, catoni2007pac, germain2016pac, Alquier2016a}. Further work explores data-dependent priors \citep{parrado12pac, dziugaite2017computing} or extends previous bounds to the scenario where priors are meta-learned \citep{pentina2014pac, amit2017meta}. However, these meta-generalization bounds are hard to optimize as they leave both the hyper-posterior and posterior unspecified, which leads to difficult nested optimization problems. In contrast, our bounds also hold for unbounded losses and yield a {\em tractable meta-learning objective} without the reliance on nested optimization.

\vspacecaption
\section{Background: PAC-Bayesian Framework}
\vspacecaptionlow
\label{sec:background}

\textbf{Preliminaries and notation.}
\looseness -1 A learning task is characterized by an unknown data distribution $\calD$ over a domain $\calZ$ from which we are given a set of $m$ observations $S = \left\{ z_i \right\}_{i=1}^m$, $z_i \sim \calD$. By $S \sim \calD^m$ we denote the i.i.d.\ sampling of $m$ data points.
%
\looseness -1 In supervised learning, we typically consider pairs $z_i = (x_i, y_i)$, where $x_i \in \calX $ are observed input features and $y_i \in \calY$ are target labels.
Given a sample $S$, our goal is to find a hypothesis $h \in \calH$, typically a function $h: \calX \rightarrow \calY$ in some hypothesis space $\calH$, that enables us to make predictions for new inputs $x^* \sim \calD_x$.
The quality of the predictions is measured by a {\em loss function} $l: \calH \times \calZ \rightarrow \R$. Accordingly, we want to minimize the {\em expected error} under the data distribution, that is, $\calL(h,\calD) = \expect{z^* \sim \calD}{l(h, z^*)}$. Since $\calD$ is unknown, we typically use the {\em empirical error}  $\hat{\calL}(h, S) = \frac{1}{m} \sum_{i=1}^m l(h, z_i)$ instead.

%
\looseness -1 In the PAC-Bayesian framework, we consider {\em randomized predictors}, i.e., probability measures on the hypothesis space $\calH$. This allows us to reason about the predictor's (epistemic) uncertainty, resulting from the fact that only a finite number of data points are available for training. We consider two such probability measures, the \emph{prior} $P \in \calM(\calH)$ and the \emph{posterior} $Q \in \calM(\calH)$. Here, $\calM(\calH)$ denotes the set of all probability measures on $\calH$. While in Bayesian inference, the prior and posterior are tightly connected through Bayes' theorem, the PAC-Bayesian framework makes fewer assumptions and only requires the prior to be independent of the observed data \citep{guedj2019primer}.
In the following, we assume that the Kullback-Leibler (KL) divergence $D_{KL}\left( Q \| P \right)$ exists. Based on the error definitions above, we can define the so-called \emph{Gibbs error} for a randomized predictor $Q$ as $\calL(Q, \calD) = \expect{h \sim Q}{\calL(h, \calD)}$ and its empirical counterpart 
as $\hat{\calL}(Q, S) = \expect{h \sim Q}{\hat{\calL}(h, S)}$.

\textbf{PAC-Bayesian bounds.}
\looseness -1 In practice, $\calL(Q, \calD)$ is unknown. Thus, one typically resorts to  minimizing $\hat{\calL}(Q, S)$ instead.
However, this may result in overfitting and poor generalization. Naturally, the question arises whether we can bound the unknown generalization error based on its empirical estimate. The PAC-Bayesian theory provides such a guarantee:

\begin{theorem} \citep{Alquier2016a} \label{theorem:alquier_pac_bound} 
Given a data distribution $\calD$, hypothesis space $\calH$, loss function $l(h,z)$, prior $P$, confidence level $\delta \in (0,1]$, and $\beta > 0$, with probability at least $1-\delta$ over samples $S \sim \calD^m$, we have for all $Q \in \calM(\calH)$:  \vspaceequation
\begin{equation}  \label{eq:alquier_pac_bound}
\vspaceequation
\calL(Q, \calD) \leq \hat{\calL}(Q, S) + \frac{1}{\beta} \left[ D_{KL}(Q || P) + \ln \frac{1}{\delta} + \Psi(\beta, m) \right] 
\vspace{3pt}
\end{equation}
with $
\Psi(\beta, m) = \ln \E_{P} \E_{\calD^m} \exp \left[ \beta \left( \calL(h, \calD) - \hat{\calL}(h, S) \right) \right] $

\end{theorem}
%

\looseness -1 Here, $\Psi(\beta, m)$ is a log moment generating function that quantifies how much the empirical error deviates from the Gibbs error. By making additional assumptions about the loss function $l$, we can bound $\Psi(\beta, m)$ and thereby obtain tractable bounds.
For instance,  if $l(h, z)$ is \emph{bounded} in $[a, b]$, we can use \mbox{Hoeffding's} lemma to obtain
$\Psi(\beta, m) \leq  (\beta^2 (b-a)^2) / (8 m)$. For unbounded loss functions, one commonly assumes that their moments are bounded. In particular, a loss function $l$ is considered \emph{sub-gamma} with variance factor $s^2$ and scale parameter $c$, under a prior $\pi$ and data distribution $\calD$, if it can be described by a sub-gamma random variable $V := \calL(h, \calD) - l(h,z)$. That is, its moment generating function is upper-bounded by that of a Gamma distribution $\Gamma(s, c)$. For details see \citet{Boucheron2013} and \citet{germain2016pac}. We can use the sub-gamma assumption to obtain $\Psi(\beta, m) \leq (\beta^2 s^2) / (2 m (1- \frac{c \beta}{m}))$.  
\textbf{Connections between the PAC-Bayesian framework and Bayesian inference.}
\looseness -1 Typically, we are interested in a posterior distribution $Q$ that promises us the lowest generalization error. In this sense, it seems natural to use the $Q \in \calM(\calH)$ that minimizes the bound in (\ref{eq:alquier_pac_bound}). Lemma~\ref{lemma:optimal_gibbs_posterior} gives us the closed-form solution to such a minimization problem:
\begin{lemma}\citep{catoni2007pac} \label{lemma:optimal_gibbs_posterior}
Let $\calH$ be a set, $g: \calH \rightarrow \R$ a (loss) function, and $Q \in \calM(\calH)$ and $P \in \calM(\calH)$ probability densities over $\calH$. Then, for any $\beta > 0$ and $h \in \calH$, 
\vspaceequation
\begin{equation} \label{eq:gibbs_dist}
Q^*(h) :=  \frac{P(h) e^{- \beta g(h)} }{Z_\beta} = \frac{P(h) e^{- \beta g(h)}}{\E_{h \sim P} \left[ e^{- \beta g(h)} \right]} \vspaceequation
\end{equation}
is the minimizing probability density of
\vspaceequation
\begin{equation} \label{eq:gibbs_agrmin}
\argmin_{Q \in \calM(\calH)} ~ \beta \E_{h \sim Q} \left[ g(h) \right] + D_{KL}(Q || P) \;.
\vspaceequation
\end{equation}
\end{lemma}
\looseness -1 The respective minimizing distribution is known as \emph{optimal Gibbs posterior} $Q^*$ \citep{catoni2007pac, Lever2013}.
As a direct consequence of Lemma~\ref{lemma:optimal_gibbs_posterior}, for fixed $P, S, m, \delta$, we can write the minimizer of (\ref{eq:alquier_pac_bound}) as \vspaceequation
$$
Q^*(h) =  P(h)e^{- \beta  \hat{\calL}(h,S)} / Z_\beta(S,P) 
\vspaceequation
$$
where $Z_\beta(S,P)=  \int_{\calH} P(h) e^{- \beta  \hat{\calL}(h,S)} dh$ is a normalization constant.
In a probabilistic setting, we typically use the negative log-likelihood of the data as our loss function $l(\cdot)$, that is, $l(h, z_i) := - \log p(z_i | h)$. In this case, the optimal Gibbs posterior coincides with the \emph{generalized Bayesian posterior} $Q^*(h ; P, S) = \frac{P(h) \, p(S \given h)^{\beta / m}}{Z_\beta(S,P)}$
where $Z_\beta(S,P) = \int_{\calH} P(h) \left( \prod_{j=1}^m p(z_j | h) \right)^{\beta / m} \, dh$ is called the \emph{generalized marginal likelihood} of the sample $S$ \citep{guedj2019primer}. For $\beta = m$ we recover the standard Bayesian posterior.
\vspacecaption
\section{PAC-Bayesian Bounds for Meta-Learning}
\vspacecaptionlow
\label{sec:method1}

We now present our main theoretical contributions. The corresponding proofs can be found in Appendix \ref{appendix:proofs}. An overview of our proposed framework is depicted in Figure~\ref{fig:overview}.

\looseness -1 \textbf{Meta-Learning.} In the standard learning setup (Sec.~\ref{sec:background}), we assumed that the learner has prior knowledge in the form of a prior distribution $P$. When the learner faces a new task, it uses the evidence in the form of a dataset $S$, to update the prior into a posterior $Q$. We formalize such a \emph{base learner} $Q(S, P)$ as a mapping $Q: \calZ^m \times \calM(\calH) \rightarrow \calM(\calH)$ that takes in a dataset and prior and outputs a posterior. 

In contrast, in meta-learning we aim to acquire such a prior $P$ in a {\em data-driven manner}, that is, by consulting a set of $n$ statistically related learning tasks $\{\tau_1, ..., \tau_n\}$. We follow the setup of \citet{baxter2000model} in which all tasks $\tau_i := (\calD_i, S_i)$ share the same data domain $\calZ:=\calX \times \calY$, hypothesis space $\calH$ and loss function $l(h,z)$, but may differ in their (unknown) data distributions $\calD_i$ and the number of points $m_i$ in the corresponding dataset $S_i \sim \calD_i^{m_i}$. To simplify our exposition, we assume w.l.o.g.\ that $m=m_i ~ \forall i$. Furthermore, each task $\tau_i \sim \calT$ is drawn i.i.d.\ from an \emph{environment} $\calT$, a probability distribution over data distributions and datasets. The goal is to {\em extract knowledge from the observed datasets, which can then be used as a prior for learning on new target tasks} $\tau \sim \calT$.
To extend the PAC-Bayesian analysis to the meta-learning setting, we again consider the notion of probability distributions over hypotheses. While the object of learning has previously been a hypothesis $h \in \calH$, it is now the prior distribution $P \in \calM(\calH)$.  The meta-learner presumes a {\em hyper-prior} $\calP \in \calM(\calM(\calH))$, that is, a distribution over priors $P$.  Combining the hyper-prior $\calP$ with the datasets $S_1, ..., S_n$ from multiple tasks, the meta-learner then outputs a {\em hyper-posterior} $\calQ$ over priors. Accordingly, the hyper-posterior's performance is measured via the expected Gibbs error when sampling priors $P$ from $\calQ$ and applying the base learner, the so-called \emph{transfer-error:}
\vspaceequation
\begin{equation*}
\vspaceequation 
\calL(\calQ, \calT) := \E_{P \sim \calQ} \left[ \E_{(\calD, m) \sim \calT} \left[ \E_{S \sim \calD^m} \left[ \calL(Q(S, P), \calD) \right] \right]  \right]
\end{equation*} 
While the transfer error is unknown in practice, we can estimate it using the \emph{empirical multi-task error}
\vspaceequation 
\begin{equation*}
\vspaceequation 
\hat{\calL}(\calQ, S_1, ..., S_n) := \E_{P \sim \calQ} \left[  \frac{1}{n} \sum_{i=1}^n \hat{\calL} \left( Q(S_i, P), S_i \right) \right] \;.
\end{equation*}

\textbf{PAC-Bayesian meta-learning bounds.}
We now present our first main result: An upper bound on the true transfer error $\calL(\calQ, \calT)$, in terms of the empirical multi-task error $\hat{\calL}(\calQ, S_1, ..., S_n)$ plus several tractable complexity terms.

\begin{theorem} \label{theorem:meta-learning-bound-bounded-lossfn}
Let $Q: \calZ^m \times \calM(\calH) \rightarrow \calM(\calH)$ be a base learner, $\calP \in \calM(\calM(\calH))$ some fixed hyper-prior $\calP$ and $\lambda, \beta > 0$. For any confidence level $\delta \in (0, 1]$ the inequality\vspace{-2mm}
\begin{align} \label{eq:meta-learning-bound}
\begin{split}
\calL(\calQ, \calT) \leq & ~ \hat{\calL}(\calQ, S_1, ..., S_n) + \left(\frac{1}{\lambda} + \frac{1}{n\beta}\right)  D_{KL}(\calQ||\calP) \hspace{-20pt} \\ & + \frac{1}{n} \sum_{i=1}^n \frac{1}{\beta} \E_{P \sim \calQ} \left[ D_{KL}(Q(S_i,P) || P)\right] \hspace{-10pt} \\ & + C(\delta, \lambda, \beta) 
\end{split}
\end{align}
holds uniformly over all hyper-posteriors $\calQ \in \calM(\calM(\calH))$ with probability $1-\delta$.
If the loss function is bounded, that is,  $l: \calH \times \calZ \rightarrow [a,b]$, the above inequality holds for 
$
C(\delta, \lambda, \beta) = \left( \frac{\lambda}{8n} + \frac{\beta}{8m} \right) (b - a)^2 + \frac{1}{\sqrt{n}} \ln \frac{1}{\delta}
$.
If the loss function is sub-gamma with variance factor $s_{\rom{1}}^2$ and scale parameter $c_{\rom{1}}$ for the data distributions $\mathcal{D}_i$ and $s_{\rom{2}}^2$, $c_{\rom{2}}$ for the task distribution $\cal{T}$, the inequality holds with
$
C(\delta, \lambda, \beta) =\frac{\lambda s_{\rom{2}}^2}{2n(1- (c_{\rom{2}} \lambda)/n)} + \frac{\beta s_{\rom{1}}^2}{2m(1- (c_{\rom{1}} \beta)/m) } + 
\frac{1}{\sqrt{n}} \ln \frac{1}{\delta}
$.
\end{theorem}

The proof of Theorem \ref{theorem:meta-learning-bound-bounded-lossfn} consists of three main steps (see Appx. \ref{appendix:proof_theorem_meta_pac_bound}). First, we bound the base learner's expected generalization error for each task when given a prior $P \sim \calQ$ and $m_i$ data samples (i.e., the error caused by observing only a finite number of samples per task). Second, we bound the generalization gap on the meta-level which is due to the fact that the meta-learning only receives finitely many tasks $\tau_i$ from $\calT$. In these two steps, we employ the change of measure inequality (Lemma \ref{lemma:concentration_sum_of_rvs}) to obtain an upper bound of the generalization error in terms of the empirical error, the KL-divergence and a cumulant-generating function. In the final step, we use either use the bounded loss or sub-gamma assumption, together with Markov's inequality, to bound the cumulant-generating function with high probability, giving rise to $C(\delta, \lambda, \beta)$.

\looseness -1 
For bounded losses, Theorem \ref{theorem:meta-learning-bound-bounded-lossfn} provides a structurally similar, but tighter bound than \citet{pentina2014pac}. In particular, by using an improved proof technique, we are able to avoid a union bound argument, allowing us to reduce the negative influence of the confidence parameter $\delta$. Unlike in the bound of \citet{amit2017meta}, the contribution of the KL-complexity term $D_{KL}(\mathcal{Q}||\mathcal{P})$ vanishes as $n \rightarrow \infty$. As a result, we find that the bound in Theorem \ref{theorem:meta-learning-bound-bounded-lossfn} is much tighter than \citep{amit2017meta} for most practical instantiations.
In contrast to \citet{pentina2014pac} and \citet{amit2017meta}, our theorem also provides guarantees for \emph{unbounded} loss functions under moment constraints (see Appendix \ref{appendix:proof_theorem_meta_pac_bound} for details). This makes Theorem ~\ref{theorem:meta-learning-bound-bounded-lossfn} particularly useful for probabilistic models in which the loss function coincides with the inherently unbounded negative log-likelihood.

\looseness -1 Common choices for $\lambda$ and $\beta$ are either a) $\lambda = \sqrt{n}$, $\beta = \sqrt{m}$ or b) $\lambda =n$, $\beta = m$ \citep{germain2016pac}. Choice a) yields consistent bounds, meaning that the gap between the transfer error and the bound vanishes as $n, m \rightarrow \infty$. In case of b), the bound always maintains a gap since $C(\delta, n, m)$ does not converge to zero. However, the KL terms decay faster, which can be advantageous for smaller sample sizes. For instance, despite their lack of consistency, sub-gamma bounds with $\beta = m$ have been shown to be much tighter in simple Bayesian linear regression scenarios with limited data ($m \lesssim 10^4$) \citep{germain2016pac}.



Previous work \citep{pentina2014pac, amit2017meta} proposes meta-learning algorithms that minimize uniform meta-generalization bounds like the one in Theorem~\ref{theorem:meta-learning-bound-bounded-lossfn}. However, in practice, the posterior $Q(S,P)$ is often intractable and thus the solution of a numerical optimization problem, like variational inference in case of BNNs. Hence, minimizing the bound in (\ref{eq:meta-learning-bound}) turns into a difficult two-level optimization problem which becomes nearly infeasible to solve for rich hypothesis spaces such as neural networks.

While Theorem~\ref{theorem:meta-learning-bound-bounded-lossfn} holds for any base learner $Q(S,P)$,  we would preferably want to use a base learner that gives us {\em optimal performance guarantees}. As discussed in Sec.~\ref{sec:background}, the Gibbs posterior not only minimizes PAC-Bayesian error bounds, but also generalizes the Bayesian posterior. Assuming a Gibbs posterior as base learner, the bound in (\ref{eq:meta-learning-bound}) can be restated in terms of the partition function $Z_\beta(S_i, P)$:

\begin{corollary}
\label{cor:bayesian_learner_PAC_bound}
When choosing a Gibbs posterior $Q^*(S_i, P)\\ := P(h) \exp (- \beta \hat{\calL}(S_i,h)) / Z_\beta(S_i, P)$ as a base learner, under the same assumptions as in Theorem~\ref{theorem:meta-learning-bound-bounded-lossfn}, we have \vspaceequation 
\begin{align} 
\begin{split} \label{eq:meta-level_pac_bound_with_mll}
\calL(\calQ, \calT) & \leq ~   - \frac{1}{n} \sum_{i=1}^n \frac{1}{\beta} \E_{P \sim \calQ} \left[\ln Z_\beta(S_i, P) \right]  \\ & + \left(\frac{1}{\lambda} + \frac{1}{n\beta}\right)  D_{KL}(\calQ||\calP) + C(\delta, \lambda, \beta) \;. \hspace{-5pt} \vspaceequation
\end{split}
\end{align}
with probability at least $1-\delta$.  \vspaceequation
\end{corollary}

\looseness -1 Since this bound assumes a \emph{PAC-optimal base learner}, it is tighter than the bound in (\ref{eq:meta-learning-bound}), which holds for any (potentially sub-optimal) $Q \in \calM(\calH)$. More importantly, (\ref{eq:meta-level_pac_bound_with_mll}) avoids the explicit dependence on $Q(S, P)$, {\em turning the previously mentioned two-level optimization problem into a standard stochastic optimization problem}. Moreover, if we choose the negative log-likelihood as the loss function and $\lambda=n, \beta_i=m_i$, then $\ln Z_\beta(S_i, P)$ {\em coincides with the marginal log-likelihood}, which is tractable for various popular learning models, such as Gaussian processes. 

The bound in Corollary \ref{cor:bayesian_learner_PAC_bound} consists of the expected generalized marginal log-likelihood under the hyper-posterior $\calQ$ as well as the KL-divergence term, which serves as a {\em regularizer on the meta-level}. As the number of training tasks $n$ grows, the relative weighting of the KL term shrinks. This is consistent with the general notion that regularization should be strong if only little data is available and vanish asymptotically as $n, m \rightarrow \infty$. 
%


\looseness -1 \textbf{Meta-Learning the hyper-posterior.} A natural way to obtain a PAC-Bayesian meta-learning algorithm could be to minimize the bound in Corollary \ref{cor:bayesian_learner_PAC_bound} w.r.t. $\calQ$. However, we can even derive the closed-form solution of such PAC-Bayesian meta-learning problem, that is, the minimizing  hyper-posterior $\calQ^*$.
For that, we exploit once more the insight that the minimizer of (\ref{eq:meta-level_pac_bound_with_mll}) can be written as a Gibbs distribution (cf. Lemma~\ref{lemma:optimal_gibbs_posterior}), giving us the following result:

\begin{proposition} \label{proposition:pacoh_optimal_hyper_posterior}
\textbf{(PAC-Optimal Hyper-Posterior)} Given a hyper-prior $\calP$ and datasets $S_1, ..., S_n$, the hyper-posterior minimizing the meta-learning bound in (\ref{eq:meta-level_pac_bound_with_mll}) is given by 
\begin{equation} \label{equation:pacoh_dl_optimal_hyper_posterior}
  \calQ^*(P) = \frac{\calP(P) \exp \left( \frac{\lambda}{n\beta + \lambda}\sum_{i=1}^n \ln Z_\beta(S_i, P) \right) }{Z^{\rom{2}}(S_1, ..., S_n, \calP)} 
\end{equation}
with
$
Z^{\rom{2}} = \E_{P \sim \calP} \left[ \exp \left(  \frac{\lambda}{n\beta + \lambda} \sum_{i=1}^n \ln Z_\beta(S_i, P) \right)  \right] \;.
$
\end{proposition}

This gives us a tractable expression for the {\em PACOH} $\calQ^*(P)$ up to the (level-\rom{2}) partition function $Z^{\rom{2}}$, which is constant with respect to $P$. We refer to $\calQ^*$ as PAC-optimal, as it provides the best possible meta-generalization guarantees among all meta-learners in the sense of Theorem \ref{theorem:meta-learning-bound-bounded-lossfn}.

\vspacecaption
\section{Meta-Learning using the PACOH}
\vspacecaptionlow
\label{sec:method2}
\looseness -1 After having introduced the closed-form solution of the PAC-Bayesian meta-learning problem in Sec.~\ref{sec:method1}, we now discuss how to translate the \emph{PACOH} into a practical meta-learning algorithm when employing GPs and BNNs as base learners.

\vspacesubcaption
\subsection{Approximating the PACOH}
\vspacesubcaption
\looseness -1 Given the hyper-prior and (level-I) log-partition function $\ln Z(S_i, P)$, we can compute the PACOH $\calQ^*$ up to the normalization constant $Z^{\rom{2}}$. Such a setup lends itself to approximate inference methods \citep{Blei2016}. In particular, we employ \emph{Stein Variational Gradient Descent (SVGD)} \citep{Liu2016} which approximates $\calQ^*$ as a set of particles $\hat{\calQ} = \{P_{\phi_1}, ..., P_{\phi_K}\}$\footnote{Note that any other approximate inference can be employed instead. We chose SVGD as we found it to work best in practice.}. Here, $P_\phi$ denotes a prior with parameters $\phi$. 
%
Alg.~\ref{algo:pacoh_generic} summarizes the resulting generic meta-learning procedure. Initially, we sample $K$ particles $\phi_k \sim \calP$ from the hyper-posterior. For notational brevity, we stack the particles into a $K \times \text{dim}(\phi)$ matrix $\bm{\phi} :=[\phi_1, ..., \phi_K]^\top$.
Then, in each iteration, we estimate the score of $\calQ^*$, \vspaceequation
\begin{align*} 
\nabla_{\phi_k} \ln \calQ^*(\phi_k) = \nabla_{\phi_k} \ln \calP(\phi_k) + \frac{\lambda}{n\beta + \lambda} \sum_{i=1}^n  \nabla_{\phi_k} \ln Z_{i,k}    \vspaceequation
\end{align*}
wherein $\ln Z_{i,k} = \ln Z_\beta (S_i, P_{\phi_k})$, and update the particle matrix using the SVGD update rule:   \vspaceequation
\begin{equation}
    \bm{\phi} \leftarrow \bm{\phi} + \eta ~ \mathbf{K} ~ \nabla_{\bm{\phi}} \mathbf{ln} \tilde{\calQ}^* + \nabla_{\bm{\phi}} \mathbf{K} \;; \vspaceequation
\end{equation}
where $\nabla_{\bm{\phi}} \mathbf{ln} \tilde{\calQ}^*:= [\nabla_{\phi_1} \ln \calQ^*(\phi_1), ..., \nabla_{\phi_K} \ln \calQ^*(\phi_K) ]^\top$ denotes the matrix of stacked score gradients, $\mathbf{K} := [k(\phi_k, \phi_{k'})]_{k,k'}$ the kernel matrix induced by a kernel function $k(\cdot,\cdot)$ and $\eta$ the step size for the SVGD updates.

To this point, how to parametrize the prior $P_\phi$ and how to estimate the generalized marginal log-likelihood $\ln Z_{i,k} = \ln Z_\beta(S_i, P_{\phi_k})$ (\texttt{MLL\_Estimator}) in Alg.~\ref{algo:pacoh_generic} have remained unspecified. In the following two subsections, we discuss these components in more detail for GPs and BNNs.

\begin{algorithm}[t]
\caption{PACOH with SVGD approximation of $\calQ^*$}
\label{algo:pacoh_generic}
\begin{algorithmic}
\STATE \textbf{Input:} hyper-prior $\calP$, datasets $S_1, ..., S_n$
\STATE \textbf{Input:} SVGD kernel function $k(\cdot, \cdot)$, step size $\eta$
\STATE $\bm{\phi} :=[\phi_1, ..., \phi_K] \text{ , with  } \phi_k \sim \calP$  \hfill // init prior particles
\WHILE{ not converged} 
 	\FOR{$k=1,...,K$} 
 	    \FOR{$i=1,...,n$} 
 	        \STATE $\ln Z_{i,k} \leftarrow$ \texttt{MLL\_Estimator}$(S_i, \phi_k, \beta)$
 	    \ENDFOR
     	\STATE $  \nabla_{\phi_k} \ln \tilde{\calQ}^*\leftarrow \nabla_{\phi_k} \ln \calP + \frac{\lambda}{\lambda + n\beta}  \sum_{i=1}^{n} \nabla_{\phi_k} \ln Z_{i,k}$
     \ENDFOR
	\STATE $\bm{\phi} \leftarrow \bm{\phi} + \eta ~ \mathbf{K} ~ \nabla_{\bm{\phi}} \mathbf{ln} \tilde{\calQ}^* + \nabla_{\bm{\phi}} \mathbf{K}$  \hfill // SVGD update
\ENDWHILE
\STATE \textbf{Output:} set of priors $\{P_{\phi_1}, ..., P_{\phi_K}\}$ 
\end{algorithmic}
\end{algorithm} 
\vspacesubcaption
\subsection{Meta-Learning Gaussian Process Priors}
\vspacesubcaption
\looseness -1 
\textbf{Setup.} In GP regression, each data point corresponds to a tuple $z_{i,j} = (x_{i,j},y_{i,j})$. For the $i$-th dataset, we write $S_i = (\bX_i, \by_i)$, where $\bX_i = (x_{i,1}, ..., x_{i,m_i})^\top$ and $\by_i = (y_{i,1}, ..., y_{i,m_i})^\top$. GPs are a Bayesian method in which the prior $P_\phi(h) = \mathcal{GP} \left(h| m_\phi(x), k_\phi(x, x') \right)$ is specified by a kernel $k_\phi: \calX \times \calX \rightarrow \R$ and a mean function $m_\phi: \calX \rightarrow \R$.
Similar to \citet{wilson2016deep} and \citet{fortuin2019deep}, we instantiate $m_\phi$ and $k_\phi$ as neural networks, and aim to meta-learn the parameter vector $\phi$.
Moreover, we use $\lambda = n$, $\beta = m$, the negative log-likelihood as loss function and a Gaussian hyper-prior $\calP = \calN(0, \sigma^2_{\calP} I)$ over the GP prior parameters $\phi$. 

\textbf{Algorithm.}
\looseness -1 In our setup, $\ln Z_m(S_i, P_\phi) = \ln p(\by_i | \bX_i, \phi)$ is the marginal log-likelihood of the GP which is available in closed form. In particular, the \texttt{MLL\_Estimator} is given by (\ref{eq:mll_gp}) in Appx.~\ref{appendix:gp_regression}. Thus, the score $\nabla_{\phi} \ln \calQ^*(\phi)$ is tractable, allowing us to perform SVGD efficiently. Alg.~\ref{algo:pacoh_gp_batched} in Appx.~\ref{appendix:mll_estimate} summarizes the proposed meta-learning method which we refer to as \emph{PACOH-GP}.

\vspacesubcaption
\subsection{Meta-Learning Bayesian Neural Network Priors}
\vspacesubcaption
\looseness -1 \textbf{Setup.} Let $h_\theta: \calX \rightarrow \calY$ be a function parametrized by a neural network (NN) with weights $\theta \in \Theta$. Using the NN mapping, we define a conditional distribution $p(y|x,\theta)$. For regression, we may set $p(y|x,\theta) = \calN(y|h_\theta(x), \sigma^2)$, where $\sigma^2$ is the observation noise variance. We treat $\ln \sigma$ as a learnable parameter similar to the neural network weights $\theta$ so that a hypothesis coincides with a tuple $h = (\theta, \ln \sigma)$.
For classification, we choose $p(y|x,\theta) = \mathrm{Categorical}(\mathrm{softmax}(h_\theta(x)))$. Our loss function is the negative log-likelihood $l(\theta,z) = - \ln p(y|x, \theta)$.

\looseness -1 Next, we define a family of priors $\{P_\phi: \phi \in \Phi\}$ over the NN parameters $\theta$. For computational convenience, we employ diagonal Gaussian priors, that is, $P_{\phi_l} = \calN(\mu_{P_k}, \text{diag}(\sigma_{P_k}^2))$ with $\phi:= (\mu_{P_k}, \ln \sigma_{P_k})$. Note that we represent $\sigma_{P_k}$ in the log-space to avoid additional positivity constraints. In fact, any parametric distribution that supports re-parametrized sampling and has a tractable log-density (e.g., normalizing flows \citep[c.f.,][]{rezende2015variational}) could be used. 
Moreover, we use a zero-centered, spherical Gaussian hyper-prior $\calP := \calN(0,\sigma_{\calP}^2 I)$ over the prior parameters $\phi$.

\textbf{Approximating the marginal log-likelihood.} 
Unlike for GPs, the (generalized) marginal log-likelihood (MLL) 
\begin{equation}
     \ln Z_\beta(S_i, P_\phi) = \ln \E_{\theta \sim P_\phi} \left[  e^{- \beta_i \hat{\calL}(\theta, S_i)} \right]
\end{equation}
is intractable for BNNs. Estimating and optimizing $\ln Z_\beta(S_i, P_\phi)$ is not only challenging due to the high-dimensional expectation over $\Theta$ but also due to numerical instabilities inherent in computing $e^{- \beta_i \hat{\calL}(\theta,S_i)}$ when $\beta_i = m$ is large.
Aiming to overcome these issues, we compute numerically stable Monte Carlo estimates of $\nabla_{\phi} \ln Z_\beta(S_i, P_{\phi_k})$ by combining the LogSumExp (LSE) with the re-parametrization trick \citep{kingma2014auto}. In particular, the \texttt{MLL\_Estimator} draws $L$ samples $\theta_l := f(\phi_k, \eps_l) = \mu_{P_k} + \sigma_{P_k} \odot \eps_l, ~ \eps_l \sim N(0,I) $ and computes the generalized MLL estimate as follows: \vspaceequation
\begin{equation} \label{eq:mll_estimator}
\ln \tilde{Z}_{\beta_i}(S_i, P_\phi) := ~ \text{LSE}_{l=1}^L\left( - \beta_i  \hat{\calL}(\theta_l, S_i) \right) - \ln L  \vspaceequation
\end{equation} 
Note that $\ln \tilde{Z}_\beta(S_i, P_\phi)$ is a consistent but not an unbiased estimator of $\ln Z_\beta(S_i, P_\phi)$ (see Appx. \ref{appendix:mll_estimate} for details).

\looseness -1 \textbf{Algorithm.} Alg. \ref{algo:pacoh_nn} in Appx. \ref{appendix:mll_estimate} summarizes the proposed meta-learning method which we henceforth refer to as \emph{PACOH-NN}. To estimate the score $ \nabla_{\phi_{k'}} \ln \calQ^*(\phi_{k'})$, we can even use mini-batching on the task level. This mini-batched version, outlined in Alg.~\ref{algo:pacoh_nn_batched}, maintains $K$ particles to approximate the hyper-posterior, and in each forward step samples $L$ NN-parameters (of dimensionaly $|\Theta|$) per prior particle, that are deployed on a mini-batch of $n_{bs}$ tasks, to estimate the score of $\calQ^*$. As a result, the total space complexity is on the order of $\mathcal{O}(|\Theta|K + L)$ and the computational complexity of the algorithm for a single iteration is $\mathcal{O}(K^2 + K L n_{bs})$.

\looseness -1 A key advantage of \emph{PACOH-NN} over previous methods for meta-learning BNN priors \citep[e.g.,][]{pentina2014pac, amit2017meta} is that it {\em turns the previously nested optimization problem into a much simpler standard stochastic optimization problem}. This makes meta-learning not only much more stable but also more scalable. In particular, we do not need to explicitly compute the task posteriors $Q_i$ and can employ mini-batching over tasks. Thus, the computational and space complexities {\em do not} depend on the number of tasks $n$. In comparison, \emph{MLAP} \citep{amit2017meta} has a memory footprint of $\mathcal{O}(|\Theta|n)$ making meta-learning prohibitive even for moderately many (e.g., 50) tasks.


\label{sec:experiments}
\begin{figure}[t]
\begin{subfigure}{0.49\textwidth}
    \centering
    \includegraphics[width=1.0\textwidth]{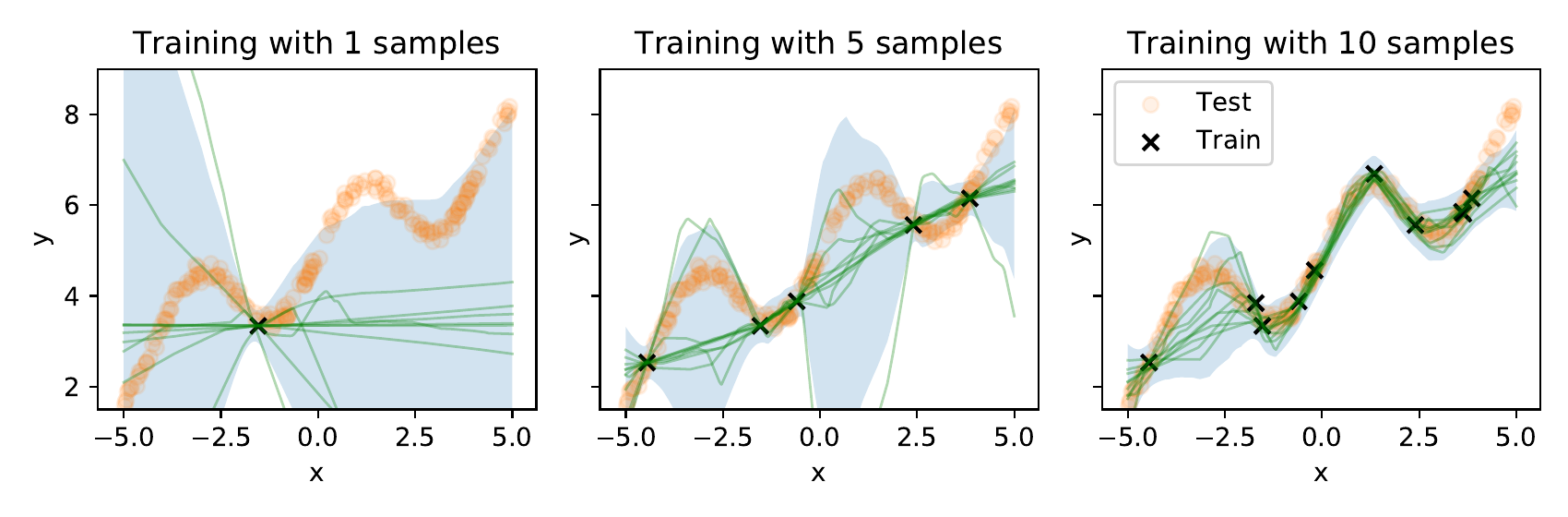}
        \vspace{-12pt}
\end{subfigure}%
\hfill
\begin{subfigure}{0.49\textwidth}
    \centering
        \vspace{-4pt}
    \includegraphics[trim={0 0.0cm 0 0.0cm}, width=1.00\textwidth]{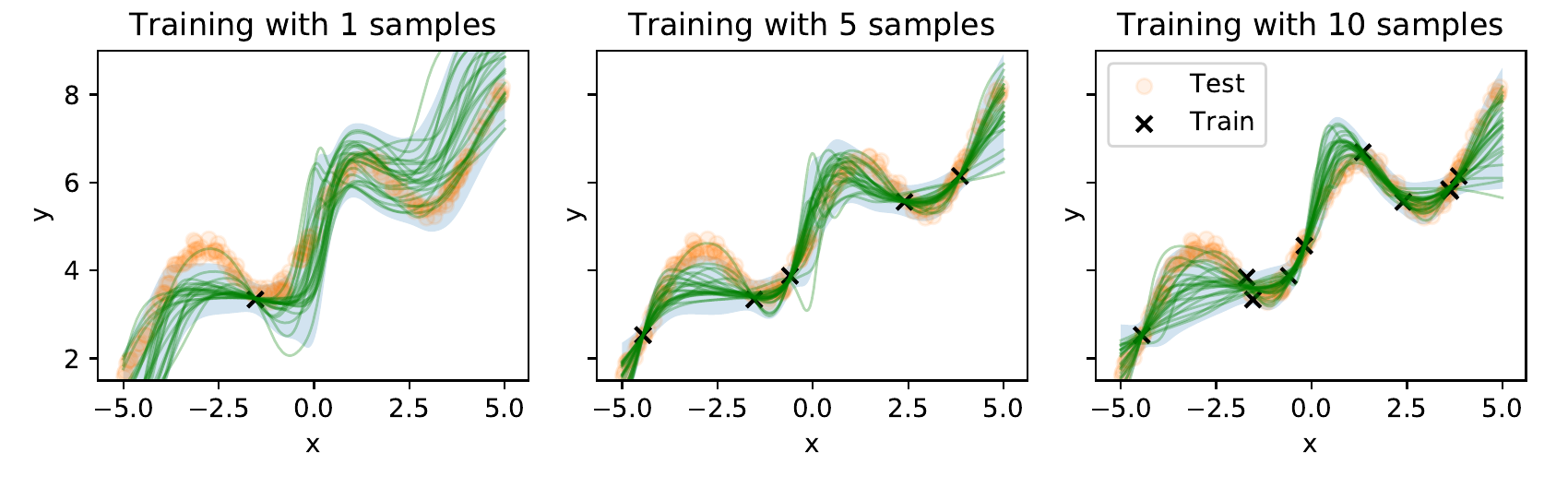} 
    \vspace{-12pt}
\end{subfigure}%
\vspace{-8pt}
\caption{BNN posterior predictions with (top) standard Gaussian prior vs. (bottom) meta-learned  prior. Meta-learning with \emph{PACOH-NN} was conducted on the \emph{Sinusoids} environment.}
\label{fig:sine_ilustrations}
\vspace{-8pt}
\end{figure}
\vspacecaption
\section{Experiments}
\vspacecaptionlow

We now empirically evaluate the two methods \emph{PACOH-GP}\footnote{The source code for \emph{PACOH-GP} is available at  \href{https://tinyurl.com/pacoh-gp-code}{\texttt{tinyurl.com/pacoh-gp-code}}.} and \emph{PACOH-NN}\footnote{An implementation of \emph{PACOH-NN} can be found at
\href{https://tinyurl.com/pacoh-nn-code}{\texttt{tinyurl.com/pacoh-nn-code}}.} that were introduced in Section~\ref{sec:method2}.
\looseness -1 Comparing them to existing meta-learning approaches on various regression and classification environments, we demonstrate that our \emph{PACOH}-based methods (i) outperform previous meta-learning algorithms in {\em predictive accuracy}, (ii) improve the calibration of {\em uncertainty estimates}, (iii) are much more {\em scalable} than previous PAC-Bayesian meta-learners, and (iv) effectively combat {\em meta-overfitting}. Finally, we showcase how meta-learned \emph{PACOH-NN} priors can be harnessed in a real-world {\em sequential decision making} task concerning peptide-based vaccine development.

\begin{table*}[t]
\centering
\resizebox{0.9 \textwidth}{!}{\begin{tabular}{l||c|c|c|c|c}
& \multicolumn{1}{c|}{Cauchy} & \multicolumn{1}{c|}{SwissFel} & \multicolumn{1}{c|}{Physionet-GCS} & \multicolumn{1}{c|}{Physionet-HCT} & \multicolumn{1}{c}{Berkeley-Sensor} \\
\hline
Vanilla GP & $0.275 \pm 0.000$ & $0.876 \pm 0.000$ & $2.240 \pm 0.000$ & $2.768 \pm 0.000 $ & $0.276 \pm 0.000 $ \\
Vanilla BNN \citep{Liu2016} & $0.327 \pm 0.008$ & $0.529 \pm 0.022$ & $2.664 \pm 0.274$ & $3.938 \pm 0.869$ & $0.109 \pm 0.004$ \\
\hline
MLL-GP \citep{fortuin2019deep} & $0.216 \pm 0.003$ & $0.974 \pm 0.093$ & $1.654 \pm 0.094$ & $2.634 \pm 0.144$ & $0.058 \pm 0.002$  \\
MLAP \citep{amit2017meta}& $0.219 \pm 0.004$ & $0.486 \pm 0.026$ & $2.009 \pm 0.248$ & $2.470 \pm 0.039$ & $0.050 \pm 0.005$ \\
MAML \citep{finn2017model} & $0.219 \pm 0.004$ & $0.730 \pm 0.057$ & $1.895 \pm 0.141$ & $2.413 \pm 0.113$ & $\mathbf{0.045 \pm 0.003}$ \\
BMAML \citep{kim2018bayesian}  & $0.225 \pm 0.004$ & $0.577 \pm 0.044$ & $1.894 \pm 0.062$ & $2.500 \pm 0.002$ & $0.073 \pm 0.014$ \\
NP \citep{garnelo2018neural}  & $0.224 \pm 0.008$ & $0.471 \pm 0.053$ & $2.056 \pm 0.209$ & $2.594 \pm 0.107$ & $0.079 \pm 0.007$ \\  \hline
PACOH-GP (ours) & $0.209  \pm 0.008$ & $0.376 \pm 0.024$ & $\mathbf{1.498 \pm 0.081}$ & $\mathbf{2.361 \pm 0.047}$ & $0.065 \pm 0.005$ \\ 
PACOH-NN (ours) & $\mathbf{0.195 \pm 0.001}$ & $\mathbf{0.372 \pm 0.002}$ & $\mathbf{1.561 \pm 0.061}$ & $\mathbf{2.405 \pm 0.017}$ & $\mathbf{0.043 \pm 0.001}$ \\ \hline

\end{tabular}}
\vspace{-4pt}
\caption{Comparison of standard and meta-learning algorithms in terms of test RMSE in 5 meta-learning environments for regression. Reported are mean and standard deviation across 5 seeds. Our proposed method \emph{PACOH} achieves the best performance across all tasks.  \vspace{-4pt}
} \label{tab:reg_rmse}
\end{table*}

\begin{table*}[t]
\centering
\resizebox{0.9\textwidth}{!}{
\begin{tabular}{l||c|c|c|c|c}
& \multicolumn{1}{c|}{Cauchy} & \multicolumn{1}{c|}{SwissFel} & \multicolumn{1}{c|}{Physionet-GCS} & \multicolumn{1}{c|}{Physionet-HCT} & \multicolumn{1}{c}{Berkeley-Sensor} \\
\hline
Vanilla GP & $0.087 \pm 0.000$ & $0.135 \pm 0.000$ & $0.268 \pm 0.000$ & $\mathbf{0.277 \pm 0.000}$ & $0.119 \pm 0.000 $ \\
Vanilla BNN  \citep{Liu2016} & $0.055 \pm 0.006$ & $0.085 \pm 0.008$ & $0.277 \pm 0.013$ & $0.307 \pm 0.009$ & $0.179 \pm 0.002$ \\
\hline
MLL-GP \citep{fortuin2019deep} & $0.059 \pm 0.003$ & $0.096 \pm 0.009$ & $ 0.277 \pm 0.009$ & $0.305 \pm 0.014$ & $0.153 \pm 0.007$  \\
MLAP \citep{amit2017meta} & $0.086 \pm 0.015$ & $0.090 \pm 0.021$ & $0.343 \pm 0.017$ & $0.344 \pm 0.016$ & $0.108 \pm 0.024$ \\
BMAML \citep{kim2018bayesian}  & $0.061 \pm 0.007$ & $0.115 \pm 0.036 $ & $0.279 \pm  0.010$ & $0.423 \pm 0.106$ & $0.161 \pm 0.013$ \\ 
NP \citep{garnelo2018neural} & $0.057 \pm 0.009$ & $0.131 \pm 0.056$ & $0.299 \pm 0.012$ & $0.319 \pm 0.004$ & $0.210 \pm 0.007$ \\ \hline
PACOH-GP (ours) & $0.056 \pm 0.004$ & $0.038 \pm 0.006$ & $\mathbf{0.262 \pm  0.004}$ & $0.296 \pm 0.003$ & $0.098 \pm 0.005$ \\ 
PACOH-NN (ours) & $\mathbf{0.046 \pm 0.001}$ & $\mathbf{0.027 \pm 0.003}$ & $0.267 \pm 0.005$ & $0.302 \pm 0.003$ & $\mathbf{0.067 \pm 0.005}$ \\
\hline
\end{tabular}
}
\vspace{-4pt}
\caption{\looseness -1 Comparison of standard and meta-learning methods in terms of the test calibration error in 5 regression environments. We report the mean and standard deviation across 5 random seeds. PACOH yields the best uncertainty calibration in the majority of environments. \vspace{-8pt}} \label{table:cal_err_reg}
\end{table*}

\vspacesubcaption
\subsection{Experiment Setup} \label{sec:exp_setup}
\vspacesubcaption
\textbf{Regression environments.} 
In our experiments, we consider two synthetic and four real-world meta-learning environments for \emph{regression}. As a synthetic environment we employ \emph{Sinusoids} of varying amplitude, phase and slope as well as a 2-dimensional mixture of \emph{Cauchy} distributions plus random GP functions. 
As real-world environments, we use datasets corresponding to different calibration sessions of the Swiss Free Electron Laser (\emph{SwissFEL}) \citep{milne2017swissfel, kirschner2019swissfel}, as well as data from the \emph{PhysioNet} 2012 challenge, which consists of time series of electronic health measurements from intensive care patients \citep{silva2012predicting}, in particular the Glasgow Coma Scale (\emph{GCS}) and the hematocrit value (\emph{HCT}). Here, the different tasks correspond to different patients. Moreover, we employ the Intel Berkeley Research Lab temperature sensor dataset (\emph{Berkeley-Sensor}) \citep{intel_sensor_data} where the tasks require auto-regressive prediction of temperature measurements corresponding to sensors installed in different locations of the building. Further details can be found in Appendix~\ref{appendix:meta-envs}.

\textbf{Classification environments.}
We conduct experiments with the multi-task \emph{classification} environment \emph{Omniglot} \citep{lake2015human}, consisting of handwritten letters across 50 alphabets. Unlike previous work \citep[e.g.,][]{finn2017model} we do not perform data augmentation and do not recombine letters of different alphabets, thus preserving the data's original structure and mitigating the need for prior knowledge. In particular, one task corresponds to 5-way 5-shot classification of letters within an alphabet. This leaves us with much fewer tasks (30 meta-train, 20 meta-test tasks), making the environment more challenging and interesting for uncertainty quantification. This also allows us to include \emph{MLAP} in the experiment which hardly scales to more than 50 tasks. 

 \textbf{Baselines.} 
We use a \emph{Vanilla GP} with squared exponential kernel and a \emph{Vanilla BNN} with a zero-centered, spherical Gaussian prior and SVGD posterior inference \citep{Liu2016} as baselines.
Moreover, we compare our proposed approach against various popular meta-learning algorithms, including model-agnostic meta-learning (\emph{MAML})~\citep{finn2017model}, Bayesian MAML (\emph{BMAML})~\citep{kim2018bayesian} and the PAC-Bayesian approach by \citet{amit2017meta} (\emph{MLAP}). For regression experiments, we also report results for neural processes (\emph{NPs}) \citep{garnelo2018neural} and a GP with neural-network-based mean and kernel function, meta-learned by maximizing the marginal log-likelihood (\emph{MLL-GP})~\citep{fortuin2019deep}.
Among all, \emph{MLAP} is the most similar to our approach as it is neural-network-based and minimizes PAC-Bayesian bounds on the transfer error. Though, unlike \emph{PACOH-NN}, it relies on nested optimization of the task posteriors $Q_i$ and the hyper-posterior $\calQ$. \emph{MLL-GP}  is similar to \emph{PACOH-GP} insofar that it also maximizes the sum of marginal log-likelihoods $\ln Z_m(S_i, P_\phi)$ across tasks. However, unlike \emph{PACOH-GP} it lacks any form of meta-level regularization.

\vspacesubcaption
\subsection{Experiment Results} \label{sec:exp_reg_class}
\vspacesubcaption

\textbf{Qualitative example.} Figure \ref{fig:sine_ilustrations} illustrates \emph{BNN} predictions on a sinusoidal regression task with a standard Gaussian prior as well as a \emph{PACOH-NN} prior meta-learned with 20 tasks from the \textit{Sinusoids} environment. We can see that the standard Gaussian prior provides poor inductive bias, not only leading to bad mean predictions away from the test points but also to poor 95\% confidence intervals (blue shaded areas). In contrast, the meta-learned \emph{PACOH-NN} prior encodes useful inductive bias towards sinusoidal function shapes, leading to better predictions and uncertainty estimates, even with minimal training data.

\textbf{PACOH improves the predictive accuracy.}
\looseness -1 Using the meta-learning environments and baseline methods that we introduced in Sec. \ref{sec:exp_setup}, we perform a comprehensive benchmark study. Table~\ref{tab:reg_rmse} reports the results on the regression environments in terms of the root mean squared error (RMSE) on unseen test tasks. Among the approaches, \emph{PACOH-NN} and \emph{PACOH-GP} {\em consistently perform best or are among the best methods}. Similarly, \emph{PACOH-NN} achieves the {\em highest accuracy} in the Omniglot classification environment (cf. Table \ref{tab:classification}). Overall, this demonstrates that the introduced meta-learning framework is not only sound, but also yields state-of-the-art empirical performance in practice.

%
%

\begin{table}[t]
\centering
\resizebox{\linewidth}{!}{
\begin{tabular}{l||c|c}
& Accuracy & Calibration error \\ \hline
Vanilla BNN  \citep{Liu2016}& $0.795 \pm 0.006$   & $0.135 \pm 0.009$  \\ \hline

MLAP \citep{amit2017meta} &  $0.700 \pm 0.0135 $  & $0.108 \pm 0.010$ \\ 

MAML \citep{finn2017model}  &  $0.693 \pm 0.013$ & $0.109 \pm 0.011$   \\
BMAML \citep{kim2018bayesian}  & $0.764 \pm 0.025$  & $0.191 \pm 0.018$  \\ \hline
PACOH-NN (ours) & $\mathbf{0.885\pm 0.090}$ & $\mathbf{0.091 \pm 0.010}$ \\ \hline
\end{tabular}}
\vspace{-8pt}
\caption{Comparison of meta-learning algorithms in terms of test accuracy and calibration error on the \emph{Omniglot} environment. Among the methods, PACOH-NN makes the most accurate and best-calibrated class predictions.\vspace{-8pt}
}
\label{tab:classification}
\end{table}

\textbf{PACOH improves the predictive uncertainty.}
\looseness -1 We hypothesize that by acquiring the prior in a principled data-driven manner (e.g., with \emph{PACOH}), we can improve the quality of the GP's and BNN's uncertainty estimates. To investigate the effect of meta-learned priors on the uncertainty estimates of the base learners, we compute the probabilistic predictors' calibration errors, reported in Table~\ref{table:cal_err_reg} and~\ref{tab:classification}. The {\em calibration error} measures the discrepancy between the predicted confidence regions and the actual frequencies of test data in the respective areas~\citep{Kuleshov2018}. Note that, since \emph{MAML} only produces point predictions, the concept of calibration does not apply to it. We observe that meta-learning priors with \emph{PACOH-NN} consistently improves the Vanilla BNN's uncertainty estimates. Similarly, \emph{PACOH-GP} yields a lower calibration error than the Vanilla GP in the majority of the envionments. For meta-learning environments where the task similarity is high, like \emph{SwissFEL} and \emph{Berkeley-Sensor}, the improvement is substantial. 

\looseness -1 \textbf{PACOH is scalable.} Unlike \emph{MLAP}~\citep{amit2017meta}, \emph{PACOH-NN} does not need to maintain posteriors $Q_i$ for the meta-training tasks and can use mini-batching on the task level. As a result, it is {\em computationally much faster and more scalable} than previous PAC-Bayesian meta-learners. This is reflected in its computation and memory complexity, discussed in Section~\ref{sec:method2}. Figure~\ref{fig:complexity_analysis} showcases this computational advantage during meta-training with \emph{PACOH-NN} and \emph{MLAP} on the \emph{Sinusoids} environment with varying number of tasks, reporting the maximum memory requirements, as well as the training time. While \emph{MLAP's} memory consumption and compute time grow linearly and become prohibitively large even for less than 100 tasks, \emph{PACOH-NN} maintains a constant memory and compute load as the number of tasks grow.

\begin{figure}
\centering
\vspace{-4pt}
\includegraphics[width=\linewidth]{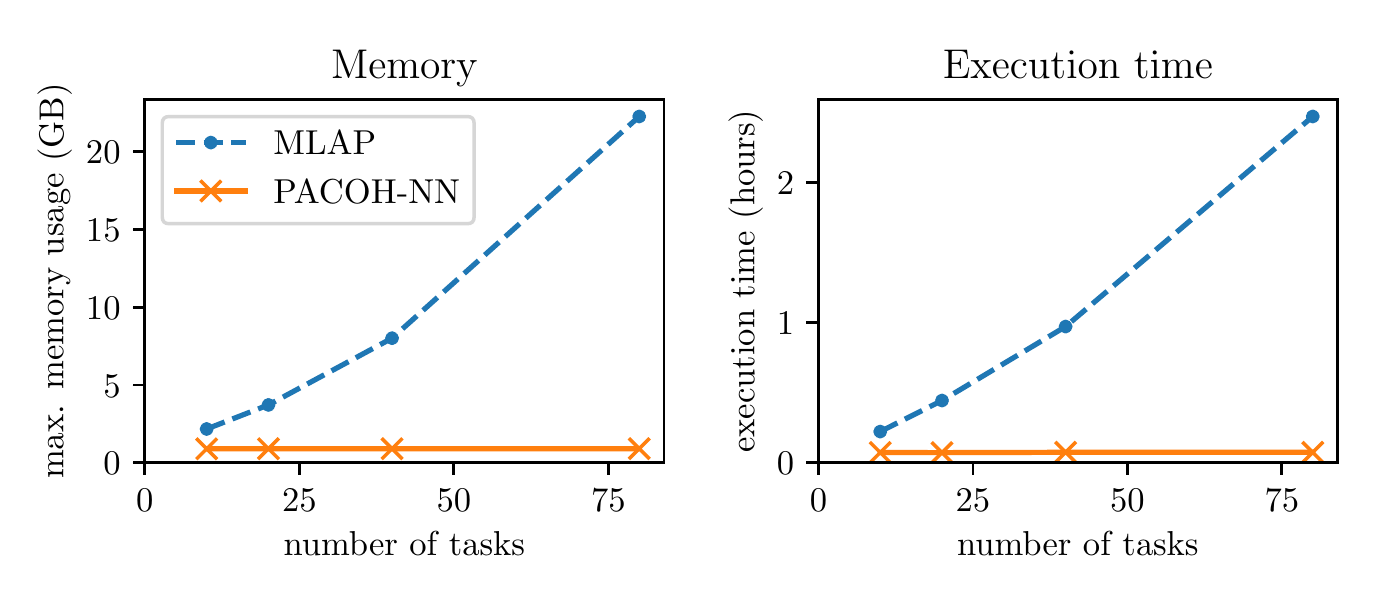}
\vspace{-24pt}
\caption{Comparison of \emph{PACOH-NN} and \emph{MLAP} in memory footprint and compute time, as the number of meta-training task grows. \emph{PACOH-NN} scales much better in the number of tasks than \emph{MLAP}. \vspace{-18pt}}
\label{fig:complexity_analysis}
\end{figure}

\looseness -1 \textbf{PACOH combats meta-overfitting.} 
As \citet{qin2018rethink} and \citet{yin2020meta} point out, many  popular meta-learners \citep[e.g.,][]{finn2017model, garnelo2018neural} require a large number of meta-training tasks to generalize well. When presented with only a limited number of tasks, such algorithms suffer from severe meta-overfitting, adversely impacting their performance on unseen tasks from $\calT$. This can even lead to {\em negative transfer}, such that meta-learning actually hurts the performance when compared to standard learning. In our experiments, we also observe such failure cases: For instance, in the classification environment (cf. Table \ref{tab:classification}), \emph{MAML} fails to improve upon the Vanilla BNN. Similarly, in the regression environments (cf. Table \ref{tab:classification}) we find that \emph{NPs}, \emph{BMAML} and \emph{MLL-GP} often yield worse-calibrated predictive distributions than the Vanilla BNN and GP respectively.
In contrast, thanks to its theoretically principled construction, \emph{PACOH-NN} is able to achieve positive transfer even when the tasks are diverse and small in number. In particular, the hyper-prior acts as meta-level regularizer by penalizing complex priors that are unlikely to convey useful inductive bias for unseen learning tasks.

\looseness -1 To investigate the importance of meta-level regularization through the hyper-prior in more detail, we compare the performance of our proposed method \emph{PACOH-GP} to \emph{MLL-GP} \citep{fortuin2019deep} which also maximizes the sum of GP marginal log-likelihoods across tasks but has no hyper-prior nor meta-level regularization. 
Fig.~\ref{fig:meta_overfitting} shows that MLL-GP performs significantly better on the meta-training tasks than on the meta-test tasks in both of our synthetic regression environments. This gap between meta-train performance and meta-test performance signifies overfitting on the meta-level.
In contrast, our method hardly exhibits this gap and consistently outperforms MLL-GP. As expected, this effect is particularly pronounced when the number of meta-training tasks is small (i.e., less than 100) and vanishes as $n$ becomes large.
Similar results for other meta-learning methods can be found in Appendix~\ref{appendix:exp-results}.
Once more, this demonstrates the importance of meta-level regularization, and shows that our proposed framework effectively addresses the problem of meta-overfitting.

\begin{figure}
\centering
\includegraphics[width=1.0\linewidth]{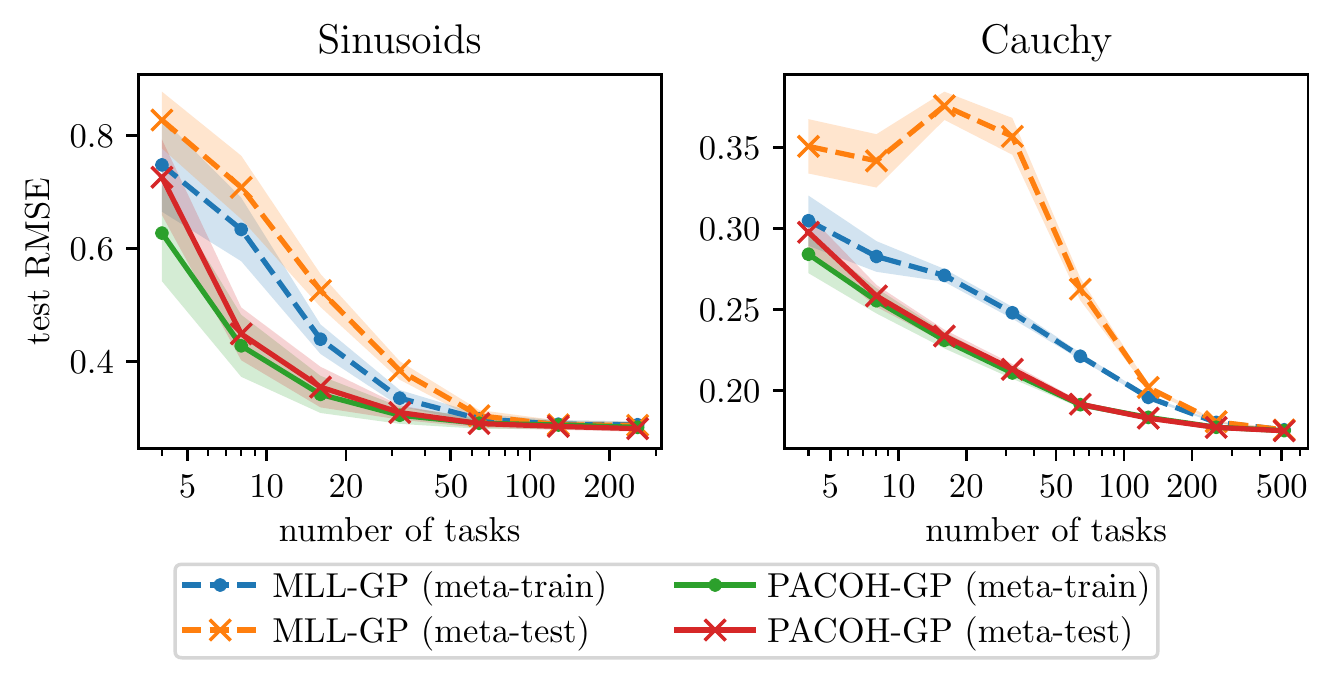}\vspace{-6pt}
\caption{ \looseness -1 Test RMSE on meta-training and meta-testing tasks as a function of the number of meta-training tasks for PACOH-GP  and MLL-GP. The performance gap between the train and test tasks demonstrates overfitting in the MLL method, while PACOH performs consistently better and barely overfits. \vspace{-8pt}} 
\label{fig:meta_overfitting}
\end{figure}

\vspacesubcaption
\subsection{Meta-Learning for Sequential Decision Making} \label{sec:bandit}
\vspacesubcaption
 

\looseness -1 Finally, we showcase how a relevant real-world application such as {\em vaccine design} can benefit from our proposed method. The goal is to discover peptides which bind to major histocompatibility complex class-I molecules (MHC-I).
Following the Bayesian optimization (BO) setup of \citet{krause2011contextual}, each task corresponds to searching for maximally binding peptides, a vital step in the design of peptide-based vaccines. The tasks differ in their targeted MHC-I allele, corresponding to different genetic variants of the MHC-I protein. We use data from \citet{widmer2010inferring}, which contains the standardized binding affinities ($\text{IC}_{50}$values) of different peptide candidates (encoded as 45-dimensional feature vectors) to the MHC-I alleles. 

\begin{figure}
\centering
\vspace{-2pt}
\includegraphics[width=\linewidth]{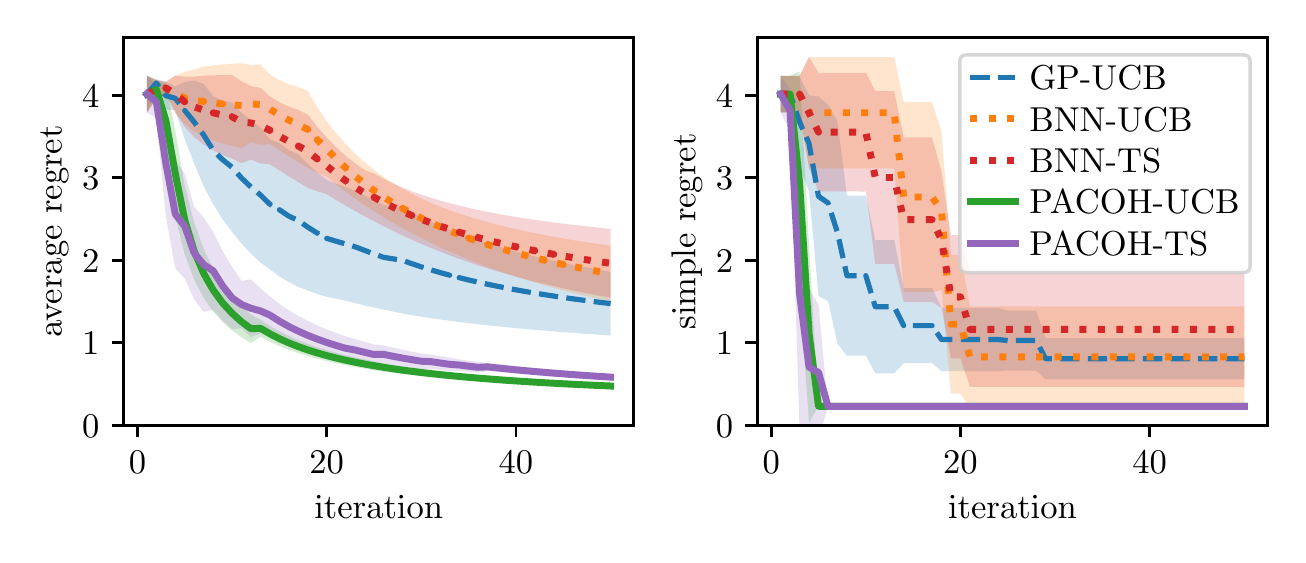}
\vspace{-24pt}
\caption{MHC-I peptide design task: Regret for different priors and bandit algorithms. A meta-learned \emph{PACOH-NN} prior substantially improves the regret, compared to a standard BNN/GP prior. \vspace{-24pt}}
 \label{fig:bandit_mhc}
\end{figure}

\looseness -1 We use 5 alleles (tasks) to meta-learn a BNN prior with \emph{PACOH-NN} and leave the most genetically dissimilar allele (A-6901) for our bandit task. In each iteration, the experimenter (i.e., the Bayesian optimization algorithm) chooses to test one peptide among the pool of more than 800 candidates and receives its binding affinity as a reward feedback. In particular, we employ UCB \citep{auer2002using} and Thompson-Sampling (TS) \citep{thompson1933likelihood} as bandit algorithms, comparing the BNN-based bandits with meta-learned prior (\emph{PACOH-UCB/TS}) against a zero-centered Gaussian BNN prior (\emph{BNN-UCB/TS}) and a Gaussian process (\emph{GP-UCB}) \citep{srinivas2009gaussian}. 

\looseness -1 Fig.~\ref{fig:bandit_mhc} reports the respective average regret and simple regret over 50 iterations.
Unlike the bandit algorithms with standard BNN/GP prior, \emph{PACOH-UCB/TS} reaches near optimal regret within less than 10 iterations and after 50 iterations still maintains a significant performance advantage. This highlights the importance of \emph{transfer (learning)} for solving real-world problems and demonstrates the effectiveness of \emph{PACOH-NN} to this end. While the majority of meta-learning methods rely on a large number of meta-training tasks \citep{qin2018rethink}, \emph{PACOH-NN} allows us to achieve promising positive transfer, even in complex real-world scenarios with only a handful (in this case 5) of tasks.
\vspacecaption
\section{Conclusion}
\vspacecaptionlow
\label{sec:conclusion}

\looseness -1 
We presented PACOH, a novel, theoretically principled, and scalable PAC-Bayesian meta-learning approach. PACOH outperforms existing methods in terms of predictive performance and uncertainty calibration, while providing PAC-Bayesian guarantees without relying on costly nested optimization.
It can be used with different base learners (e.g., GPs or BNNs) and achieves positive transfer in regression and classification with as little as five meta-tasks. BNNs meta-learnt with PACOH can be effectively used for sequential decision making, as demonstrated by our vaccine design application.
We believe our approach provides an important step towards learning useful inductive bias from data in a flexible, scalable, and  principled manner.

\section*{Acknowledgements}
\looseness -1 This project has received funding from the European Research Council (ERC) under the European Union’s Horizon 2020 research and innovation program grant agreement No 815943 and was supported by Oracle Cloud Services.
Vincent Fortuin is funded by a PhD fellowship from the Swiss Data Science Center and by the grant 2017-110 of the Strategic Focus Area ``Personalized Health
and Related Technologies (PHRT)'' of the ETH Domain.
We thank David Lindner, Gideon Dresdner, and Claire Vernade for their valuable feedback.

\bibliography{references}
\bibliographystyle{icml2021}

\onecolumn
\newpage

\beginsupplement

\appendix

\section*{Appendix}

\section{Proofs and Derivations} \label{appendix:proofs}

\subsection{Proof of Theorem~\ref{theorem:meta-learning-bound-bounded-lossfn}}
\label{appendix:proof_theorem_meta_pac_bound}

\begin{lemma} \label{lemma:concentration_sum_of_rvs}  \textbf{(Change of measure inequality)}
Let $f$ be a random variable taking values in a set $A$ and let $X_1, ..., X_l$ be independent random variables, with $X_k \in A$ with distribution $\mu_k$. For functions $g_k: A \times A \rightarrow \R, k=1,...,l$, let $\xi_k(f)=\E_{X_k \sim \mu_k} \left[g_k(f, X_k)\right]$ denote the expectation of $g_k$ under $\mu_k$ for any fixed $f \in A$. Then for any fixed distributions $\pi, \rho$ $\in \mathcal{M}(A)$ and any $\lambda > 0$, we have that
\begin{align}
\E_{f\sim\rho} \left[ \sum_{k=1}^l \xi_k(f)\ - g_k(f,X_k) \right] \leq \frac{1}{\lambda} \left( D_{KL}(\rho||\pi) + \ln \E_{f\sim\pi} \left[ e^{ \lambda \big( \sum_{k=1}^l \xi_k(f) - g_k(f,X_k) \big)}  \right] \right) .
\end{align}
\end{lemma}

To prove the Theorem \ref{theorem:meta-learning-bound-bounded-lossfn}, we need to bound the difference between \emph{transfer error} $\calL(\calQ, \calT)$ and the \emph{empirical multi-task error} $\hat{\calL}(\calQ, S_1, ..., S_n)$. To this end, we introduce an intermediate quantity, the \emph{expected multi-task error}:
\begin{equation}
\tilde{\calL}(\calQ, \calD_1, ..., \calD_n) = \E_{P \sim \calQ} \left[ \frac{1}{n} \sum_{i=1}^n \E_{S \sim D_i^{m_i}} \left[ \calL(Q(S, P), \calD_i) \right] \right]
\end{equation}
In the following we invoke Lemma \ref{lemma:concentration_sum_of_rvs} twice. First, in step 1, we bound the difference between $\tilde{\calL}(\calQ, \calD_1, ..., \calD_n)$ and $\hat{\calL}(\calQ, S_1, ..., S_n)$, then, in step 2, the difference between $\calL(\calQ, \calT)$ and $\tilde{\calL}(\calQ, \calD_1, ..., \calD_n)$. Finally, in step 3, we use a union bound argument to combine both results.

\paragraph{Step 1 (Task specific generalization)}
First, we bound the generalization error of the observed tasks $\tau_i=(\calD_i, m_i)$, $i=1,...,n$, when using a learning algorithm $Q: \calM \times \calZ^{m_i} \rightarrow \calM$, which outputs a posterior distribution $Q_i=Q(S_i, P)$ over hypotheses $h$, given a prior distribution $P$ and a dataset $S_i \sim \calD^{m_i}_i$ of size $m_i$. In that, we define $\overline{m} := (\sum_{i=1}^n m_i^{-1})^{-1}$ as the harmonic mean of sample sizes.

In particular, we apply Lemma~\ref{lemma:concentration_sum_of_rvs} to the union of all training sets $S' = \bigcup_{i=1}^n S_i$ with $l=\sum_{i=1}^n m_i$.  Hence, each $X_k$ corresponds to one data point, i.e. $X_k = z_{ij}$ and $\mu_k = \calD_i$. Further, we set $f=(P, h_1, ..., h_n)$ to be a tuple of one prior and $n$ base hypotheses. This can be understood as a two-level hypothesis, wherein $P$ constitutes a hypothesis of the meta-learning problem and $h_i$ a hypothesis for solving the supervised task $\tau_i$. Correspondingly, we take $\pi = (\calQ, Q^n)  = \calP \cdot \prod_{i=1}^n P$ and $\rho = (\calQ, Q^n) = \calQ \cdot \prod_{i=1}^n Q_i$ as joint two-level distributions and $g_k(f, X_k) = \frac{1}{n m_i} l(h_i, z_{ij})$ as summand of the empirical multi-task error. We can now invoke Lemma~\ref{lemma:concentration_sum_of_rvs} to obtain that (\ref{eq:step_1_bound_raw}) and (\ref{eq:two_level_kl})
%
\begin{align} \label{eq:step_1_bound_raw}
\begin{split}
\frac{1}{n} \sum_{i=1}^n \E_{P \sim \calQ} \left[ \calL(Q_i,D_i) \right] \leq &
\frac{1}{n} \sum_{i=1}^n \E_{P \sim \calQ} \left[ \calL(Q_i,S_i) \right] + 
\frac{1}{\gamma} \bigg( D_{KL}\left[ (\calQ, Q^n) || (\calP, P^n) \right] \\ 
&+\ln \E_{P \sim \calP} \E_{h \sim P}  \left[ e^{ \frac{\gamma}{n} \sum_{i=1}^n (\calL(h, \calD_i) - \hat{\calL}(h, S_i)}\right]\bigg)
\end{split}
\end{align}
Using the above definitions, the KL-divergence term can be re-written in the following way:
\begin{align}
    D_{KL}\left[ (\calQ, Q^n) || (\calP, P^n) \right]  &= \E_{P \sim \calQ} \left[ \E_{h \sim Q_i} \left[ \ln \frac{\calQ(P)  \prod_{i=1}^n  \prod Q_i(h)}{\calP(P)  \prod_{i=1}^n  P(h)}\right] \right] \\
    &= \E_{P \sim \calQ} \left[ \ln \frac{\calQ(P)}{\calP(P)}\right] + \sum_{i=1}^n \E_{P \sim \calQ} \left[ \E_{h \sim Q_i} \left[ \ln \frac{Q_i(h)}{P(h)}\right] \right] \\
    &= D_{KL}(\calQ||\calP) + \sum_{i=1}^n \E_{P \sim \calQ} \left[ D_{KL}(Q_i || P)\right] \label{eq:two_level_kl}
\end{align}
Using (\ref{eq:step_1_bound_raw}) and (\ref{eq:two_level_kl}) we can bound the expected multi-task error as follows:
\begin{align} \label{eq:step1_bound}
\begin{split}
\tilde{\calL}(\calQ, \calD_1, ..., \calD_n)  \leq &~ \hat{\calL}(\calQ, S_1, ..., S_n)  + \frac{1}{\gamma} D_{KL}(\calQ||\calP) + \frac{1}{\gamma} \sum_{i=1}^n \E_{P \sim \calQ} \left[ D_{KL}(Q_i || P)\right] \\
&+ \underbrace{ \frac{1}{\gamma} \ln \E_{P \sim \calP} \E_{h \sim P}  \left[ e^{ \frac{\gamma}{n} \sum_{i=1}^n (\calL(h, \calD_i) - \hat{\calL}(h, S_i))}\right]}_{\Upsilon^{\rom{1}}(\gamma)} \\
\end{split}
\end{align}

\paragraph{Step 2 (Task environment generalization)}
Now, we apply Lemma~\ref{lemma:concentration_sum_of_rvs} on the meta-level. For that, we treat each task as random variable and instantiate the components as $X_k = \tau_i$, $l=n$ and $\mu_k = \calT$. Furthermore, we set $\rho = \calQ$, $\pi = \calP$, $f=P$ and $g_k(f, X_k) = \frac{1}{n} \calL(Q_i, \calD_i)$. This allows us to bound the transfer error as 
\begin{equation}  \label{eq:step_2_bound}
\calL(\calQ, \calT) \leq \tilde{\calL}(\calQ, \calD_1, ..., \calD_n) + \frac{1}{\lambda} D_{KL}(\rho||\pi) +  \Upsilon^\rom{2}(\lambda)
\end{equation}
wherein $\Upsilon^\rom{2}(\lambda) = \frac{1}{\lambda} \ln \E_{P \sim \calP} \left[ e^{ \frac{\lambda}{n} \sum_{i=1}^n \E_{(D,S)  \sim \calT} \left[ \calL (Q(P, S), \calD) \right] - \calL (Q(P, S_i), \calD_i) } \right]$.

Combining (\ref{eq:step1_bound}) with (\ref{eq:step_2_bound}), we obtain
\begin{align}
\begin{split}
\calL(\calQ, \calT) \leq &~ \hat{\calL}(\calQ, S_1, ..., S_n)  + \left(\frac{1}{\lambda} + \frac{1}{\gamma} \right) D_{KL}(\calQ||\calP) \\
& +  \frac{1}{\gamma} \sum_{i=1}^n  \E_{P \sim \calQ} \left[ D_{KL}(Q_i || P)\right] + \Upsilon^{\rom{1}}(\gamma) +  \Upsilon^\rom{2}(\lambda)
\end{split}
\end{align}

\paragraph{Step 3 (Bounding the moment generating functions)}

\begin{align}
\begin{split}
e^{(\Upsilon^{\rom{1}}(\gamma) +  \Upsilon^\rom{2}(\lambda))} = & 
 \E_{P \sim \calP} \left[ e^{\frac{\lambda}{n} \sum_{i=1}^n \E_{(D,S)  \sim \calT} \left[ \calL (Q(P, S), \calD) \right] - \calL (Q(P, S_i), \calD_i )} \right]^{1/\lambda} \cdot \\
&\E_{P \sim \calP} \E_{h \sim P}  \left[ e^{ \frac{\gamma}{n} \sum_{i=1}^n (\calL(h, \calD_i) - \hat{\calL}(h, S_i)}\right]^{1/\gamma} \\
= & 
 \E_{P \sim \calP} \left[ \prod_{i=1}^n e^{\left( \frac{\lambda}{n}  \E_{(D,S)  \sim \calT} \left[ \calL (Q(P, S), \calD) \right) \right] - \calL (Q(P, S_i), \calD_i )} \right]^{1/\lambda} \cdot \\
&\E_{P \sim \calP} \E_{h \sim P}  \left[ \prod_{i}^n \prod_{i}^{m_i} e^{ \frac{\gamma}{nm_i} (\calL(h, \calD_i) - l(h_i, z_{ij}) )}\right]^{1/\gamma} \\
\end{split} \label{eq:moment_gen_function_factorized}
\end{align}

\textbf{Case I: bounded loss} 

If the loss function $l(h_i, z_{ij})$ is bounded in $[a,b]$, we can apply Hoeffding's lemma to each factor in (\ref{eq:moment_gen_function_factorized}), obtaining:
\begin{align}
e^{\Upsilon^{\rom{1}}(\gamma) +  \Upsilon^\rom{2}(\lambda)} \leq & ~ \E_{P \sim \calP} \left[ e^{\frac{\lambda^2}{8n} (b - a)^2 } \right]^{1/\lambda}\cdot \E_{P \sim \calP} \E_{h \sim P}  \left[  e^{\frac{\gamma^2}{8n\overline{m}} (b - a)^2 } \right]^{1/\gamma} \\
 = & ~ e^{\left( \frac{\lambda}{8n} + \frac{\gamma}{8n\overline{m}} \right) (b - a)^2 }
\end{align} 

Next, we factor out $\sqrt{n}$ from $\gamma$ and $\lambda$, obtaining
\begin{equation}
e^{\Upsilon^{\rom{1}}(\gamma) +  \Upsilon^\rom{2}(\lambda)} = \left( e^{ \Upsilon^{\rom{1}}(\gamma \sqrt{n}) +  \Upsilon^\rom{2}(\lambda \sqrt{n})} \right)^{\frac{1}{\sqrt{n}}}
\end{equation}
Using
\begin{equation}
\E_{\cal{T}}\E_{\calD_1} ... \E_{\calD_n} \left[e^{ \Upsilon^{\rom{1}}(\gamma \sqrt{n}) +  \Upsilon^\rom{2}(\lambda \sqrt{n})}  \right]  \leq e^{\left( \frac{\lambda}{8\sqrt{n}} + \frac{\gamma}{8 \sqrt{n}\overline{m}} \right) (b - a)^2 }
\end{equation}
we can apply Markov's inequality w.r.t. the expectations over the task distribution $\calT$ and data distributions $\calD_i$ to obtain that
\begin{equation}
\Upsilon^{\rom{1}}(\gamma) +  \Upsilon^\rom{2}(\lambda) \leq \underbrace{\frac{\lambda}{8n} (b - a)^2}_{\Psi^{\rom{1}}(\lambda)} + \underbrace{\frac{\gamma}{8n\overline{m}} (b - a)^2}_{\Psi^{\rom{2}}(\gamma)}  - \frac{1}{\sqrt{n}}\ln \delta
\end{equation}
with probability at least $1-\delta$.

\textbf{Case II: sub-gamma loss} 

First, we assume that, $\forall i=1, ..., n$ the random variables $V^{\rom{1}}_i := \calL(h, \calD_i) - l(h_i, z_{i,j})$ are {\em sub-gamma} with variance factor $s_{\rom{1}}^2$ and scale parameter $c_{\rom{1}}$ under the two-level prior $(\calP, P)$ and the respective data distribution $\calD_i$. That is, their moment generating function can be bounded by that of a Gamma distribution $\Gamma(s_{\rom{1}}^2, c_{\rom{1}})$:
\begin{equation}
\E_{z \sim \calD_i} \E_{P \sim \calP} \E_{h \sim P} \left[ e^{ \gamma (\calL(h, \calD_i) - l(h, z) )}\right] \leq  \exp \left(\frac{\gamma^2 s_{\rom{1}}^2}{2(1- c_{\rom{1}} \gamma)} \right) \quad \forall \gamma \in ( 0, 1 / c_{\rom{1}} )
\end{equation}

Second, we assume that, the random variable $V^{\rom{2}}:= \E_{(D,S)  \sim \calT} \left[ \calL (Q(P, S), \calD) \right] - \calL (Q(P, S_i), \calD_i )$ is {\em sub-gamma} with variance factor $s_{\rom{2}}^2$ and scale parameter $c_{\rom{2}}$ under the hyper-prior $\calP$ and the task distribution $\calT$. That is, its moment generating function can be bounded by that of a Gamma distribution $\Gamma(s_{\rom{2}}^2, c_{\rom{2}})$:
\begin{equation}
 \E_{(\calD, S) \sim \calT} \E_{P \sim \calP} \left[  e^{\lambda ~ \E_{(D,S)  \sim \calT} \left[ \calL (Q(P, S), \calD) \right] - \calL (Q(P, S), \calD )}\right] \leq   \exp \left(\frac{\lambda^2 s_{\rom{2}}^2}{2(1- c_{\rom{2}} \lambda)} \right) \quad \forall \lambda \in ( 0, 1 / c_{\rom{2}} )
\end{equation}

These two assumptions allow us to bound the expectation of (\ref{eq:moment_gen_function_factorized}) as follows:
\begin{align}
\E \left[ e^{\Upsilon^{\rom{1}}(\gamma) +  \Upsilon^\rom{2}(\lambda)}  \right] \leq & ~ \exp \left(\frac{\gamma s_{\rom{1}}^2}{2n\overline{m}(1- c_{\rom{1}} \gamma / (n \overline{m}) } \right) \cdot \exp \left(\frac{\lambda s_{\rom{2}}^2}{2n(1- c_{\rom{2}} \lambda / n)} \right) 
\end{align} 
Next, we factor out $\sqrt{n}$ from $\gamma$ and $\lambda$, obtaining
\begin{equation}
e^{\Upsilon^{\rom{1}}(\gamma) +  \Upsilon^\rom{2}(\lambda)} = \left( e^{ \Upsilon^{\rom{1}}(\gamma \sqrt{n}) +  \Upsilon^\rom{2}(\lambda \sqrt{n})} \right)^{\frac{1}{\sqrt{n}}}
\end{equation}
Finally, by using Markov's inequality we obtain that
\begin{equation}
\Upsilon^{\rom{1}}(\gamma) +  \Upsilon^\rom{2}(\lambda) \leq \underbrace{\frac{\gamma s_{\rom{1}}^2}{2n\overline{m}(1- c_{\rom{1}} \gamma / (n \overline{m}) }}_{\Psi^{\rom{1}}(\gamma)} + \underbrace{\frac{\lambda s_{\rom{2}}^2}{2n(1- c_{\rom{2}} \lambda / n)}}_{\Psi^{\rom{2}}(\lambda)}  - 
\frac{1}{\sqrt{n}} \ln \delta
\end{equation}
with probability at least $1-\delta$.

To conclude the proof, we choose $\gamma = n \beta$ for $\beta > 0$.

\subsection{Proof of Corollary \ref{cor:bayesian_learner_PAC_bound}}

When we choose the posterior $Q$ as the optimal Gibbs posterior $Q^*_i := Q^*(S_i, P)$, it follows that 
\begin{align}
& \hat{\calL}(\calQ, S_1, ..., S_n) + \frac{1}{n} \sum_{i=1}^n \frac{1}{\beta} \E_{P \sim \calQ} \left[ D_{KL}(Q^*_i || P)\right] \\
=~& \frac{1}{n} \sum_{i=1}^n \left( \E_{P \sim \calQ} \E_{h \sim Q^*_i} \left[ \hat{\calL}(h, S_i) \right]  + \frac{1}{\beta} \left( \E_{P \sim \calQ} \left[ D_{KL}(Q^*_i || P)\right] \right) \right) \\
=~& \frac{1}{n} \sum_{i=1}^n \frac{1}{\beta} \left(\E_{P \sim \calQ} \E_{h \sim Q^*_i} \left[ \beta \hat{\calL}(h, S_i) +  \ln \frac{Q^*_i(h)}{P(h)} \right] \right) \label{eq:task_objective_in_bound} \\
=~& \frac{1}{n} \sum_{i=1}^n \frac{1}{\beta} \left(\E_{P \sim \calQ} \E_{h \sim Q^*_i} \left[  \beta \hat{\calL}(h, S_i)  +  \ln \frac{P(h) e^{-  \beta \hat{\calL}(h, S_i) }}{P(h) Z_\beta(S_i, P)} \right] \right) \\
=~& \frac{1}{n} \sum_{i=1}^n \frac{1}{\beta} \left(- \E_{P \sim \calQ} \left[ \ln Z_\beta(S_i, P)\right] \right) \;.
\end{align}

This allows us to write the inequality in (\ref{eq:meta-learning-bound}) as 
\begin{align} \label{eq:meta_pac_bound_mll_appendix}
\calL(\calQ, \calT)  \leq ~  & - \frac{1}{n} \sum_{i=1}^n \frac{1}{\beta} \E_{P \sim \calQ} \left[\ln Z_\beta(S_i, P) \right]  + \left(\frac{1}{\lambda} + \frac{1}{n\beta}\right)  D_{KL}(\calQ||\calP) + C(\delta, n, \overline{m}) \;.
\end{align}

According to Lemma~\ref{lemma:optimal_gibbs_posterior}, the Gibbs posterior $Q^*(S_i, P)$ is the minimizer of (\ref{eq:task_objective_in_bound}), in particular $\forall P \in \calM(\calH), \forall i=1, ..., n:$
\begin{equation}
Q^*(S_i, P) = \frac{P(h)e^{- \beta \hat{\calL}(h,S_i)}}{Z_\beta(S_i,P)} = \argmin_{Q \in \calM(\calH)} \E_{h \sim Q} \left[ \hat{\calL}(h, S_i) \right]  + \frac{1}{\beta} D_{KL}(Q || P)  \;.
\end{equation}

Hence, we can write
\begin{align}
\calL(\calQ, \calT)  \leq & - \frac{1}{n} \sum_{i=1}^n \frac{1}{\beta} \E_{P \sim \calQ} \left[\ln Z_\beta(S_i, P) \right]  + \left(\frac{1}{\lambda} + \frac{1}{n\beta}\right)  D_{KL}(\calQ||\calP) + C(\delta, \lambda, \beta) \\
= & ~ \frac{1}{n} \sum_{i=1}^n \E_{P \sim \calQ} \left[\min_{Q \in \calM(\calH)} \hat{\calL}(Q, S_i)  + \frac{1}{\beta}  D_{KL}(Q || P) \right] \\
& + \left(\frac{1}{\lambda} + \frac{1}{n\beta}\right)  D_{KL}(\calQ||\calP) + C(\delta, n, \overline{m}) \\
\leq & \frac{1}{n} \sum_{i=1}^n \E_{P \sim \calQ} \left[\hat{\calL}(Q, S_i)  + \frac{1}{\beta}  D_{KL}(Q || P) \right] \\ & + \left(\frac{1}{\lambda} + \frac{1}{n\beta}\right)  D_{KL}(\calQ||\calP) + C(\delta, \lambda, \beta) \\
= & ~ \hat{\calL}(\calQ, S_1, ..., S_n) +\left(\frac{1}{\lambda} + \frac{1}{n\beta}\right) D_{KL}(\calQ||\calP) \\ &+  \frac{1}{n} \sum_{i=1}^n \frac{1}{\beta} \E_{P \sim \calQ} \left[ D_{KL}(Q_i || P)\right] + C(\delta, \lambda, \beta) \;,
\end{align}
which proves that the bound for Gibbs-optimal base learners in (\ref{eq:meta_pac_bound_mll_appendix}) and (\ref{eq:meta-level_pac_bound_with_mll}) is tighter than the bound in Theorem~\ref{theorem:meta-learning-bound-bounded-lossfn} which holds uniformly for all $Q \in \calM(\calH)$.

\subsection{Proof of Proposition~\ref{proposition:pacoh_optimal_hyper_posterior}: PAC-Optimal Hyper-Posterior}
\label{appendix:proof_optimal-hyper-posterior}
An objective function corresponding to (\ref{eq:meta-level_pac_bound_with_mll}) reads as
\begin{equation} \label{eq:meta-objective_with_mll}
J(\calQ) =  - \E_{\calQ} \left[ \frac{\lambda}{n\beta + \lambda} \sum_{i=1}^n \ln Z(S_i, P) \right] + D_{KL}(\calQ||\calP) \;.
\end{equation}
To obtain $J(\calQ)$, we omit all additive terms from (\ref{eq:meta-level_pac_bound_with_mll}) that do not depend on $\calQ$ and multiply by the scaling factor $\frac{\lambda n \beta}{n \beta + \lambda}$. Since the described transformations are monotone, the minimizing distribution of $J(\calQ)$, that is,
\begin{equation}
\calQ^* = \argmin_{\calQ \in \calM(\calM(\calH))} J(\calQ) \;,
\end{equation}
is also the minimizer of (\ref{eq:meta-level_pac_bound_with_mll}). More importantly, $J(\calQ)$ is structurally similar to the generic minimization problem in (\ref{eq:gibbs_agrmin}). Hence, we can invoke Lemma~\ref{lemma:optimal_gibbs_posterior} with $A = \calM(\calH)$, $g(a) = - \sum_{i=1}^n \ln Z(S_i, P)$, $\beta = \frac{1}{\sqrt{n\overline{m}} + 1}$, to show that the optimal hyper-posterior is
\begin{equation}
\calQ^*(P) = \frac{\calP(P) \exp \left( \frac{\lambda}{n\beta + \lambda}\sum_{i=1}^n \ln Z_\beta(S_i, P) \right) }{Z^{\rom{2}}(S_1, ..., S_n, \calP)} \;,
\end{equation}
wherein 
\begin{align*}
& Z^{\rom{2}}(S_1, ..., S_n, \calP) = \E_{P \sim \calP} \left[ \exp \left( \frac{\lambda}{n\beta + \lambda} \sum_{i=1}^n \ln Z_\beta(S_i, P) \right)  \right] \;.
\end{align*} $\hfill \Box$

Technically, this concludes the proof of Proposition~\ref{proposition:pacoh_optimal_hyper_posterior}. However, we want to remark the following result:

If we choose $\calQ = \calQ^*$, the PAC-Bayes bound in (\ref{eq:meta-level_pac_bound_with_mll}) can be expressed in terms of the meta-level partition function $Z^{\rom{2}}$, that is,
\begin{align}
\calL(\calQ, \calT)  &\leq - \left(\frac{1}{\lambda} + \frac{1}{n\beta}\right) \ln Z^{\rom{2}}(S_1, ..., S_n, \calP) + C(\delta, \lambda, \beta) \;. \label{eq:pac_bound_z2}
\end{align}
We omit a detailed derivation of (\ref{eq:pac_bound_z2}) since it is similar to the one for Corollary \ref{cor:bayesian_learner_PAC_bound}.

\section{Gaussian process regression}
\label{appendix:gp_regression}

In GP regression, each data point corresponds to a feature-target tuple $z_{i,j} = (x_{i,j},y_{i,j}) \in \R^d \times \R$. For the $i$-th dataset, we write $S_i = (\bX_i, \by_i)$, where $\bX_i = (x_{i,1}, ..., x_{i,m_i})^\top$ and $\by_i = (y_{i,1}, ..., y_{i,m_i})^\top$. GPs are a Bayesian method in which the prior $P_\phi(h) = \mathcal{GP} \left(h| m_\phi(x), k_\phi(x, x') \right)$ is specified by a positive definite kernel $k_\phi: \calX \times \calX \rightarrow \R$ and a mean function $m_\phi: \calX \rightarrow \R$.

 The empirical loss under the GP posterior $Q^*$ coincides with the negative log-likelihood of regression targets $\by_i$, that is,  $\hat{\calL}(Q^*, S_i) = - \frac{1}{m_i} \ln p(\by_i|\bX_i)$. 
Under a Gaussian likelihood $p(\by|\mathbf{h}) = \calN(\by;h(\bx), \sigma^2 I)$, the marginal log-likelihood $\ln Z(S_i, P_\phi) = \ln p(\by_i | \bX_i, \phi)$ can be computed in closed form as
\begin{align} \label{eq:mll_gp}
\begin{split}
\ln p(\by | \bX, \phi) = & - \frac{1}{2} \left( \by- m_{\bX,\phi}) \right)^\top \tilde{K}_{\bX,\phi}^{-1}  \left( \by- m_{\bX,\phi} \right) - \frac{1}{2} \ln |\tilde{K}_{\bX, \phi}| - \frac{m_i}{2} \ln 2 \pi \; , 
\end{split}
\end{align}
where $\tilde{K}_{\bX, \phi} = K_{\bX, \phi} + \sigma^2 I$, with the kernel matrix $K_{\bX, \phi}=(k_\phi(x_l,  x_k))^{m_i}_{l,k=1}$, observation noise variance $\sigma^2$, and mean vector $m_{\bX,\phi} = (m_\phi(x_1),..., m_\phi(x_{m_i}))^\top$.

Previous work on Bayesian model selection in the context of GPs argues that the log-determinant $\frac{1}{2} \ln |\tilde{K}_{\bX, \phi}|$ in the marginal log-likelihood (\ref{eq:mll_gp}) acts as a complexity penalty \citep{Rasmussen2001, rasmussen2003gaussian}. However, we suspect that this complexity regularization is only effective if the class of considered priors is restrictive, for instance if we only optimize a small number of parameters such as the length- and output scale of a squared exponential kernel. If we consider expressive classes of GP priors (e.g., our setup where the mean and kernel function are neural networks), such a complexity penalty could be insufficient to avoid meta-overfitting.
Indeed, this is what we also observe in our experiments (see Sec.~\ref{sec:experiments}).


\section{PACOH-GP: Meta-Learning of GP priors} \label{appendix:pacoh}
In this section, we provide further details on PACOH-GP, introduced in Section \ref{sec:method2} of the paper and employed in our experiments. Following Section \ref{sec:experiments}, we instantiate our framework with GP base learners. Since we are interested in meta-learning, we define the mean and kernel function both as parametric functions.
Similar to \citet{wilson2016deep} and \citet{fortuin2019deep}, we instantiate $m_\phi$ and $k_\phi$ as neural networks, where the parameter vector $\phi$ can be meta-learned.
To ensure the positive-definiteness of the kernel, we use the neural network as feature map $\Phi_\phi(x)$ on top of which we apply a squared exponential (SE) kernel.
Accordingly, the parametric kernel reads as $k_\phi(x, x') = \frac{1}{2}\exp \left( - ||\Phi_\phi(x) - \Phi_\phi(x')||_2^2 \right)$. Both $m_\phi(x)$ and $\Phi_\phi(x)$ are fully-connected neural networks with 4 layers with each 32 neurons and $\tanh$ non-linearities. The parameter vector $\phi$ represents the weights and biases of both neural networks. As hyper-prior we choose a zero-mean isotropic Gaussian, that is, $\calP(\phi) = \calN(0, \sigma_{\calP}^2 I)$. Further, we choose $\lambda = n$ and $\beta_i = m_i$

\subsection{Meta-training with PACOH-GP}
SVGD \citep{Liu2016} approximates $\calQ^*$ as a set of particles $\hat{\calQ} = \{P_1, ..., P_K\}$. In our described setup, each particle corresponds to the parameters of the GP prior, that is, $\hat{\calQ} =  \{\phi_1, ..., \phi_K\}$. Initially, we sample random priors $\phi_k \sim \calP$ from our hyper-posterior. Then, the SVGD iteratively transports the set of particles to match $\calQ^*$, by applying a form of functional gradient descent that minimizes  $D_{KL}(\hat{\calQ} | \calQ^*)$ in the reproducing kernel Hilbert space induced by a kernel function $k(\cdot,\cdot)$. We choose a squared exponential kernel with length scale (hyper-)parameter $\ell$, that is, $k(\phi ,\phi') = \ \exp \left( - \frac{|| \phi - \phi'||_2^2}{2 \ell} \right)$. In each iteration, the particles are updated by 
\begin{equation*}
\phi_k \leftarrow \phi_k + \eta_t \psi^*(\phi_k) ~, \quad \text{with} \quad \psi^*(\phi) = \frac{1}{K} \sum_{l=1}^K \left [k(\phi_l, \phi) \nabla_{\phi_l} \ln \calQ^*(\phi_l) +  \nabla_{\phi_l} k(\phi_l, \phi) \right] \;.
\end{equation*}
Here, we can again estimate $\nabla_{\phi_l} \ln \calQ^*(\phi_l)$ with a mini-batch of $H$ datasets $S_1, ..., S_H$:
$$
\nabla_{\phi_l} \ln \calQ^*(\phi_l) = \frac{n}{H} \cdot \sum_{h=1}^H \frac{1}{m_h + 1} \nabla_{\phi_l} \ln Z(S_h, P_{\phi_l})  + \nabla_{\phi_l} \ln \calP(\phi_l) \;.
$$ 
Importantly, $\nabla_{\phi_l} \ln \calQ^*(\phi_l)$ does not depend on $Z^{\rom{2}}$ which makes SVGD tractable. Algorithm \ref{algo:pacoh_gp_batched} summarizes the meta-training method for GP priors.

\begin{algorithm}
\caption{PACOH-GP: mini-batched meta-training}
\label{algo:pacoh_gp_batched}
\begin{algorithmic}
\STATE \textbf{Input:} hyper-prior $\calP$, datasets $S_1, ..., S_n$
\STATE \textbf{Input:} SVGD kernel function $k(\cdot, \cdot)$, SVGD step size $\eta$, number of particles $K$
\STATE $\{ \phi_1, ..., \phi_K\} \sim \calP$ \hfill // Initialize prior particles
\WHILE{not converged} 
    \STATE $\{T_1, ..., T_{n_{bs}}\} \subseteq [n]$ \hfill // sample $n_{bs}$ tasks uniformly at random	
 	\FOR{$k=1,...,K$} 
     	\FOR{$i=1, ..., n_{bs}$}
     	    \STATE $\ln Z_{m_i}(S_i, P_{\phi_k}) \leftarrow - \frac{1}{2} \left( \by_i- m_{\bX_i,\phi_k} \right)^\top \tilde{K}_{\bX_i,\phi_k}^{-1}  \left( \by_i- m_{\bX_i,\phi_k} \right)  - \frac{1}{2} \ln |\tilde{K}_{\bX_i, \phi_k}| - \frac{m_i}{2} \ln 2 \pi $ \hfill // compute MLL
     	\ENDFOR
     	\STATE $  \nabla_{\phi_k} \ln \tilde{\calQ}^*(\phi_k)  \leftarrow \nabla_{\phi_k} \ln \calP(\phi_k) + \frac{n}{n_{bs}}\sum_{i=1}^{n_{bs}} \frac{1}{m_i + 1} \nabla_{\phi_k} \ln Z_{m_i}(S_i, P_{\phi_k})$ \hfill // compute score
	\ENDFOR
	
	\STATE $\phi_k \leftarrow \phi_k + \frac{\eta}{K} \sum_{k'=1}^K \left [k(\phi_{k'}, \phi_k) \nabla_{\phi_{k'}} \ln \tilde{\calQ}^*(\phi_{k'}) +  \nabla_{\phi_{k'}} k(\phi_{k'}, \phi_k) \right] \forall{k \in [K]}$ \hfill // SVGD update
\ENDWHILE
\STATE \textbf{Output:} set of GP priors $\{ \mathcal{GP} \left(m_{\phi_1}(x), k_{\phi_1}(x, x') \right), ...,  \mathcal{GP} \left(m_{\phi_K}(x), k_{\phi_K}(x, x') \right) \}$
\end{algorithmic}
\end{algorithm}

\subsection{Meta-Testing / Target-Training with PACOH-GP}
Meta-learning with PACOH, as described above, gives us an approximation of $\calQ^*$. In target-testing (see Figure \ref{fig:overview}), the base learner is instantiated with the meta-learned prior $P_\phi$, receives a dataset $\tilde{S}=(\tilde{\bX}, \tilde{\by})$ from an unseen task $\calD \sim \calT$ and outputs a posterior $Q$ as product of its inference. In our GP setup, $Q$ is the GP posterior and the predictive distribution $\hat{p}(y^*|x^*, \tilde{\bX}, \tilde{\by}, \phi)$ is a Gaussian \citep[for details see][]{rasmussen2003gaussian}.

Since the meta-learner outputs a distribution over priors, that is, the hyper-posterior $\calQ$, we may obtain different predictions for different priors $P_\phi \sim \calQ$, sampled from the hyper-posterior. To obtain a predictive distribution under our meta-learned hyper-posterior, we empirically marginalize $\calQ$. That is, we draw a set of prior parameters $\phi_1, ..., \phi_K \sim \calQ$ from the hyper-posterior, compute their respective predictive distributions $\hat{p}(y^*|x^*, \tilde{\bX}, \tilde{\by}, \phi_k)$ and form an equally weighted mixture:
\begin{equation} \label{eq:predictive_mixture}
    \hat{p}(y^*|x^*, \tilde{\bX}, \tilde{\by}, \calQ) = \E_{\phi \sim \calQ} \left[ \hat{p}(y^*|x^*, \tilde{\bX}, \tilde{\by}, \phi) \right] \approx \frac{1}{K} \sum_{k=0}^{K} \hat{p}(y^*|x^*, \tilde{\bX}, \tilde{\by}, \phi_k) ~, \quad \phi_k \sim \calQ 
\end{equation}
Since we are concerned with GPs, (\ref{eq:predictive_mixture}) coincides with a mixture of Gaussians. As one would expect, the mean prediction under $\calQ$ (i.e., the expectation of (\ref{eq:predictive_mixture})), is the average of the mean predictions corresponding to the sampled prior parameters $\phi_1, ..., \phi_K$. In case of PACOH-VI, we sample $K=100$ priors from the variational hyper-posterior $\tilde{\calQ}$. For PACOH-SVGD, samples from the hyper-posterior correspond to the $K=10$ particles. PACOH-MAP can be viewed as a special case of SVGD with $K=1$, that is, only one particle. Thus, $\hat{p}(y^*|x^*, \tilde{\bX}, \tilde{\by}, \calQ) \approx \hat{p}(y^*|x^*, \tilde{\bX}, \tilde{\by}, \phi^{MAP})$ is a single Gaussian.

\section{PACOH-NN: Meta-Learning BNN priors} \label{appendix:mll_estimate}

In this section, we summarize and further discuss our proposed meta-learning algorithm \emph{PACOH-NN}. An overview of our proposed framework is illustrated in Figure \ref{fig:overview}. Overall, it consists of two stages \emph{meta-training} and \emph{meta-testing} which we explain in more details in the following.
\subsection{Meta-training} The hyper-posterior distribution $\calQ$ that minimizes the upper bound on the transfer error is given by
\begin{equation}
   \calQ^*(P) =  \frac{\calP(P) \exp\left( \sum_{i=1}^n  \frac{\lambda}{n\beta_i + \lambda} \ln \tilde{Z}(S_i, P)  \right) }{Z^{\rom{2}}(S_1, ..., S_n, \calP) } \vspaceequation
\end{equation}
In that, we no longer assume that $m = m_i ~ \forall i=1,...,n$ which was done in the theory section to maintain notational brevity. Thus, we use a different $\beta_i$ for each och the tasks as we want to set $\beta_i = m_i$ or $\beta_i = \sqrt{m_i}$.
Provided with a set of datasets $S_1, ..., S_n$, the meta-learner minimizes the respective meta-objective, in the case of \emph{PACOH-SVGD}, by performing SVGD on the $\calQ^*$. Algorithm \ref{algo:pacoh_nn} outlines the required steps in more detail.
\begin{algorithm}
\caption{PACOH-NN: meta-training}
\label{algo:pacoh_nn}
\begin{algorithmic}
\STATE \textbf{Input:} hyper-prior $\calP$, datasets $S_1, ..., S_n$, kernel $k(\cdot, \cdot)$, step size $\eta$, number of particles $K$
\STATE $\{ \phi_1, ..., \phi_K\} \sim \calP$ \hfill // Initialize prior particles
\WHILE{not converged} 
 	\FOR{$k=1,...,K$} 
     	\STATE $\{\theta_1, ..., \theta_L\} \sim P_{\phi_k}$ \hfill // sample NN-parameters from prior
     	
     	\FOR{$i=1, ..., n$}
         	\STATE $\ln \tilde{Z}(S_i, P_{\phi_k}) \leftarrow \text{LSE}_{l=1}^L\left( - \beta_i \hat{\calL}(\theta_l,S_i) \right) - \ln L$ \hfill // estimate generalized MLL
     	\ENDFOR
     	
     	\STATE $  \nabla_{\phi_k} \ln \tilde{\calQ}^*(\phi_k)  \leftarrow \nabla_{\phi_k} \ln \calP(\phi_k) + \sum_{i=1}^n \frac{\lambda}{n\beta_i + \lambda}   \nabla_{\phi_k} \ln \tilde{Z}(S_i, P_{\phi_k})$ \hfill // compute score
     	
	\ENDFOR
	     		\STATE $\forall k \in [K]: \phi_k \leftarrow \phi_k + \frac{\eta}{K} \sum_{k'=1}^K \left [k(\phi_{k'}, \phi_k) \nabla_{\phi_{k'}} \ln \tilde{\calQ}^*(\phi_{k'}) +  \nabla_{\phi_{k'}} k(\phi_{k'}, \phi_k) \right] $  \hfill // SVGD update

\ENDWHILE
\STATE \textbf{Output:} set of priors $\{ P_{\phi_1}, ..., P_{\phi_K} \}$
\end{algorithmic}
\end{algorithm}

Alternatively, to estimate the score of $\nabla_{\phi_k} \tilde{\calQ}^*(\phi_k)$ we can use mini-batching at both the task and the dataset level. Specifically, for a given meta-batch size of $n_{bs}$ and a batch size of $m_{bs}$, we get Algorithm \ref{algo:pacoh_nn_batched}.

\begin{algorithm}
\caption{PACOH-NN-SVGD: mini-batched meta-training}
\label{algo:pacoh_nn_batched}
\begin{algorithmic}
\STATE \textbf{Input:} hyper-prior $\calP$, datasets $S_1, ..., S_n$
\STATE \textbf{Input:} kernel function $k(\cdot, \cdot)$, SVGD step size $\eta$, number of particles $K$
\STATE $\{ \phi_1, ..., \phi_K\} \sim \calP$ \hfill // Initialize prior particles
\WHILE{not converged} 
    \STATE $\{T_1, ..., T_{n_{bs}}\} \subseteq [n]$ \hfill // sample $n_{bs}$ tasks uniformly at random
 	\FOR{$i=1, ..., n_{bs}$}
 	    \STATE $\tilde{S}_i \leftarrow \{z_1, ..., z_{m_{bs}}\} \subseteq S_{T_i}$ \hfill // sample $m_{bs}$ datapoints from $S_{T_i}$ uniformly at random
 	\ENDFOR
 	
 	\FOR{$k=1,...,K$} 
     	\STATE $\{\theta_1, ..., \theta_L\} \sim P_{\phi_k}$ \hfill // sample NN-parameters from prior
     	\FOR{$i=1, ..., n_{bs}$}
     	    \STATE $\ln \tilde{Z}(\tilde{S}_i, P_{\phi_k}) \leftarrow \text{LSE}_{l=1}^L\left( - \beta_i  \hat{\calL}(\theta_l, \tilde{S}_i) \right) - \ln L$ \hfill // estimate generalized MLL
     	\ENDFOR
     	\STATE $  \nabla_{\phi_k} \ln \tilde{\calQ}^*(\phi_k)  \leftarrow \nabla_{\phi_k} \ln \calP(\phi_k) + \frac{n}{n_{bs}}\sum_{i=1}^{n_{bs}} \frac{\lambda}{n\beta_i + \lambda} \nabla_{\phi_k} \ln \tilde{Z}(S_i, P_{\phi_k})$ \hfill // compute score
	\ENDFOR
	
	\STATE $\phi_k \leftarrow \phi_k + \frac{\eta}{K} \sum_{k'=1}^K \left [k(\phi_{k'}, \phi_k) \nabla_{\phi_{k'}} \ln \tilde{\calQ}^*(\phi_{k'}) +  \nabla_{\phi_{k'}} k(\phi_{k'}, \phi_k) \right] \forall{k \in [K]}$ \hfill // SVGD update
\ENDWHILE
\STATE \textbf{Output:} set of priors $\{ P_{\phi_1}, ..., P_{\phi_K} \}$
\end{algorithmic}
\end{algorithm}

\subsection{Meta-testing} The meta-learned prior knowledge is now deployed by a base learner. The base learner is given a training dataset $\tilde{S} \sim \calD$ pertaining to an unseen task $\tau = (\calD, m) \sim \calT$. With the purpose of approximating the generalized Bayesian posterior $Q^*(S,P)$, the base learner performs (normal) posterior inference. Algorithm \ref{algo:pacoh_nn_target_training} details the steps of the approximating procedure – referred to as \emph{target training} – when performed via SVGD. For a data point $x^*$, the respective predictor outputs a probability distribution given as $\tilde{p}(y^*| x^*, \tilde{S}) \leftarrow \frac{1}{K \cdot L}\sum_{k=1}^K \sum_{l=1}^L p(y^*|h_{\theta^k_l}(x^*))$. We evaluate the quality of the predictions on a held-out test dataset $\tilde{S}^* \sim \calD$ from the same task, in a \emph{target testing} phase (see Appendix \ref{appendix:exp_methodology}).

\begin{algorithm}
\caption{PACOH-NN: meta-testing}
\label{algo:pacoh_nn_target_training}
\begin{algorithmic}
\STATE \textbf{Input:} set of priors $\{ P_{\phi_1}, ..., P_{\phi_K} \}$, target training dataset $\tilde{S}$, evaluation point $x^*$
\STATE \textbf{Input:} kernel function $k(\cdot, \cdot)$, SVGD step size $\nu$, number of particles $L$
\FOR{$k=1, ..., K$}
 	\STATE $\{\theta^k_1, ..., \theta^k_L\} \sim P_{\phi_k}$ \hfill // initialize NN posterior particles from $k$-th prior
 	\WHILE{not converged} 
 		\FOR{$l=1,...,L$}
 		\STATE $  \nabla_{\theta^k_l} Q^*(\theta^k_l))  \leftarrow \nabla_{\theta^k_l} \ln P_{\phi_k}(\theta^k_l)) + \beta ~ \nabla_{\theta^k_l}  \calL(l, \tilde{S})$ \hfill // compute score
 		\ENDFOR
 		\STATE $\theta^k_l \leftarrow \theta^k_l +  \frac{\nu}{L} \sum_{l'=1}^L \left [k(\theta^k_{l'}, \theta^k_l) \nabla_{\theta^k_{l'}} \ln Q^{*}(\theta^k_{l'}) +  \nabla_{\theta^k_{l'}} k(\theta^k_{l'}, \theta^k_l) \right] \forall{l \in [L]}$ \hfill // update particles
 	\ENDWHILE
\ENDFOR
\STATE \textbf{Output:} a set of NN parameters
$\bigcup_{k=1}^K \{\theta_1^k\, ..., \theta_L^k\}$
\end{algorithmic}
\end{algorithm}

\subsection{Properties of the score estimator} \label{appendix:properties_mll_estimator}

Since the marginal log-likelihood of BNNs is intractable, we have replaced it by a numerically stable Monte Carlo estimator $\ln \tilde{Z}_{\beta}(S_i, P_\phi)$ in (\ref{eq:mll_estimator}), in particular
\begin{equation}
    \ln \tilde{Z}_\beta(S_i, P_\phi) := \ln  \frac{1}{L} \sum_{l=1}^{L} e^{- \beta  \hat{\calL}(\theta_l,S_i)} = \text{LSE}_{l=1}^L\left( - \beta \hat{\calL}(\theta_l, S_i) \right) - \ln L ~, ~~ \theta_l \sim P_\phi ~.
\end{equation}
Since the Monte Carlo estimator involves approximating an expectation of an exponential, it is not unbiased. However, we can show that replacing $\ln Z_{\beta}(S_i, P_\phi)$ by the estimator $\ln \tilde{Z}_{\beta}(S_i, P_\phi)$, we still minimize a valid upper bound on the transfer error (see Proposition \ref{proposition:mll_estimate_still_upper_bound}).

\begin{proposition} \label{proposition:mll_estimate_still_upper_bound}
In expectation, replacing $\ln Z_{\beta}(S_i, P_\phi)$ in (\ref{eq:meta-level_pac_bound_with_mll}) by the Monte Carlo estimate $\ln \tilde{Z}_{\beta}(S_i, P) := \ln  \frac{1}{L} \sum_{l=1}^{L} e^{- \beta \hat{\calL}(\theta_l,S_i)}, ~ \theta_l \sim P$ still yields an valid upper bound of the transfer error. In particular, it holds that
\begin{align}
\calL(\calQ, \calT) & \leq   - \frac{1}{n} \sum_{i=1}^n \frac{1}{\beta} \E_{P \sim \calQ} \left[\ln Z(S_i, P) \right]   + \left(\frac{1}{\lambda} + \frac{1}{n \beta}\right)  D_{KL}(\calQ||\calP) + C(\delta, n, \overline{m})  \hspace{-4pt} \label{eq:normal_pac_bound} \\
\begin{split} \label{eq:estimator_upper_bound}
& \leq - \frac{1}{n} \sum_{i=1}^n \frac{1}{\beta} \E_{P \sim \calQ} \left[ \E_{\theta_1,...,\theta_L \sim P} \left[ \ln \tilde{Z}(S_i, P)  \right] \right]    \\ & \quad ~ + \left(\frac{1}{\lambda} + \frac{1}{n \beta}\right) D_{KL}(\calQ||\calP) + C(\delta, \lambda, \beta) .
\end{split}
\end{align}
\end{proposition}
\begin{proof}
Firsts, we show that:
\begin{align}
\label{eq:ineq_mll_estimator}
\E_{\theta_1,...,\theta_L \sim P} \left[\ln \tilde{Z}_{\beta}(S_i, P) \right] &= \E_{\theta_1,...,\theta_L \sim P} \left[ \ln \frac{1}{L} \sum_{l=1}^{L}  e^{- \beta \hat{\calL}(\theta_l,S_i)} \right] \nonumber \\
&\leq \ln \frac{1}{L} \sum_{l=1}^{L} \E_{\theta_l \sim P} \left[  e^{- \beta \hat{\calL}(\theta_l,S_i)} \right] \nonumber \\
&= \ln \E_{\theta \sim P} \left[ e^{- \beta \hat{\calL}(\theta,S_i)} \right] \nonumber \\
&= \ln Z_\beta(S_i, P)
\end{align}
which follows directly from Jensen's inequality and the concavity of the logarithm. Now, Proposition \ref{proposition:mll_estimate_still_upper_bound} follows directly from (\ref{eq:ineq_mll_estimator}).
\end{proof}

In fact, by the law of large numbers, it is straightforward to show that as $L \rightarrow \infty$, the $\ln \tilde{Z}(S_i, P) \xrightarrow[]{\text{a.s.}} \ln Z(S_i, P) $, that is, the estimator becomes asymptotically unbiased and we recover the original PAC-Bayesian bound (i.e. (\ref{eq:estimator_upper_bound})  $\xrightarrow[]{\text{a.s.}}$ (\ref{eq:normal_pac_bound})). Also it is noteworthy that the bound in (\ref{eq:estimator_upper_bound}) we get by our estimator is, in expectation, tighter than the upper bound when using the naive estimator 
$$
\ln \hat{Z}_\beta(S_i, P) := - \beta ~ \frac{1}{L} \sum_{l=1}^{L}   \hat{\calL}(\theta_l,S_i) \quad \theta_l \sim P_\phi
$$
which can be obtained by applying Jensen's inequality to $\ln \E_{\theta \sim P_\phi} \left[ e^{- \beta \hat{\calL}(\theta,S_i)} \right]$. In the edge case $L=1$ our LSE estimator $\ln \tilde{Z}_\beta(S_i, P)$ falls back to this naive estimator and coincides in expectation with $\E [ \ln \hat{Z}_\beta(S_i, P)] = - \beta ~ \E_{\theta \sim P}   \hat{\calL}(\theta_l,S_i)$. As a result, we effectively minimize the looser upper bound

\begin{align}
\calL(\calQ, \calT) & \leq   \frac{1}{n} \sum_{i=1}^n \E_{\theta \sim P}  \left[  \hat{\calL}(\theta,S_i) \right]  + \left(\frac{1}{\lambda} + \frac{1}{n \beta}\right)   D_{KL}(\calQ||\calP) + C(\delta, n, \overline{m}) . \\
& =  \E_{\theta \sim P}  \left[  \frac{1}{n} \sum_{i=1}^n \frac{1}{m_i} \sum_{j=1}^{m_i} - \ln p(y_{ij} | x_{ij}, \theta) \right]  + \left(\frac{1}{\lambda} + \frac{1}{n \beta}\right)   D_{KL}(\calQ||\calP) + C(\delta, n, \overline{m}) \label{eq:loose_upper_bound}
\end{align}

As we can see from (\ref{eq:loose_upper_bound}), the boundaries between the tasks vanish in the edge case of $L=1$, that is, all data-points are treated as if they would belong to one dataset. This suggests that $L$ should be chosen greater than one. In our experiments, we used $L=5$ and found the corresponding approximation to be sufficient.
%
%

\section{Experiments}\label{appendix:exps}
\subsection{Meta-Learning Environments}
\label{appendix:meta-envs}

In this section, we provide further details on the meta-learning environments used in Section~\ref{sec:experiments}.
Information about the numbers of tasks and samples in the respective environments can be found in Table~\ref{tab:num_tasks_samples}.
\begin{table}[h]
\centering
\begin{tabular}{l|ccccc}
 & Sinusoid & Cauchy & SwissFEL & Physionet & Berkeley \\ \hline
 $n$ & 20 & 20 & 5 & 100 & 36 \\
 $m_i$ & 5 & 20 & 200 & 4 - 24  & 288 \\ 
\end{tabular}
\caption{Number of tasks $n$ and samples per task $m_i$ for the different meta-learning environments.}
\label{tab:num_tasks_samples}
\end{table}
\subsubsection{Sinusoids}
Each task of the sinusoid environment corresponds to a parametric function
\begin{equation}
   f_{a, b, c, \beta} (x) = \beta * x + a * \sin(1.5 * (x - b)) + c \;,
\end{equation}
which, in essence, consists of an affine as well as a sinusoid function. Tasks differ in the function parameters $(a, b, c, \beta)$ that are sampled from the task environment $\calT$ as follows:
\begin{equation}
a  \sim \calU(0.7, 1.3), \quad  b  \sim \calN(0, 0.1^2), \quad  c  \sim \calN(5.0, 0.1^2), \quad \beta  \sim \calN(0.5, 0.2^2) \;. \label{eq:param_sampling_sinusoid}
\end{equation}
\begin{figure}[t] \label{fig:simulated_tasks}
\begin{subfigure}{0.5\textwidth}
        \centering
        \includegraphics[width=0.95\textwidth]{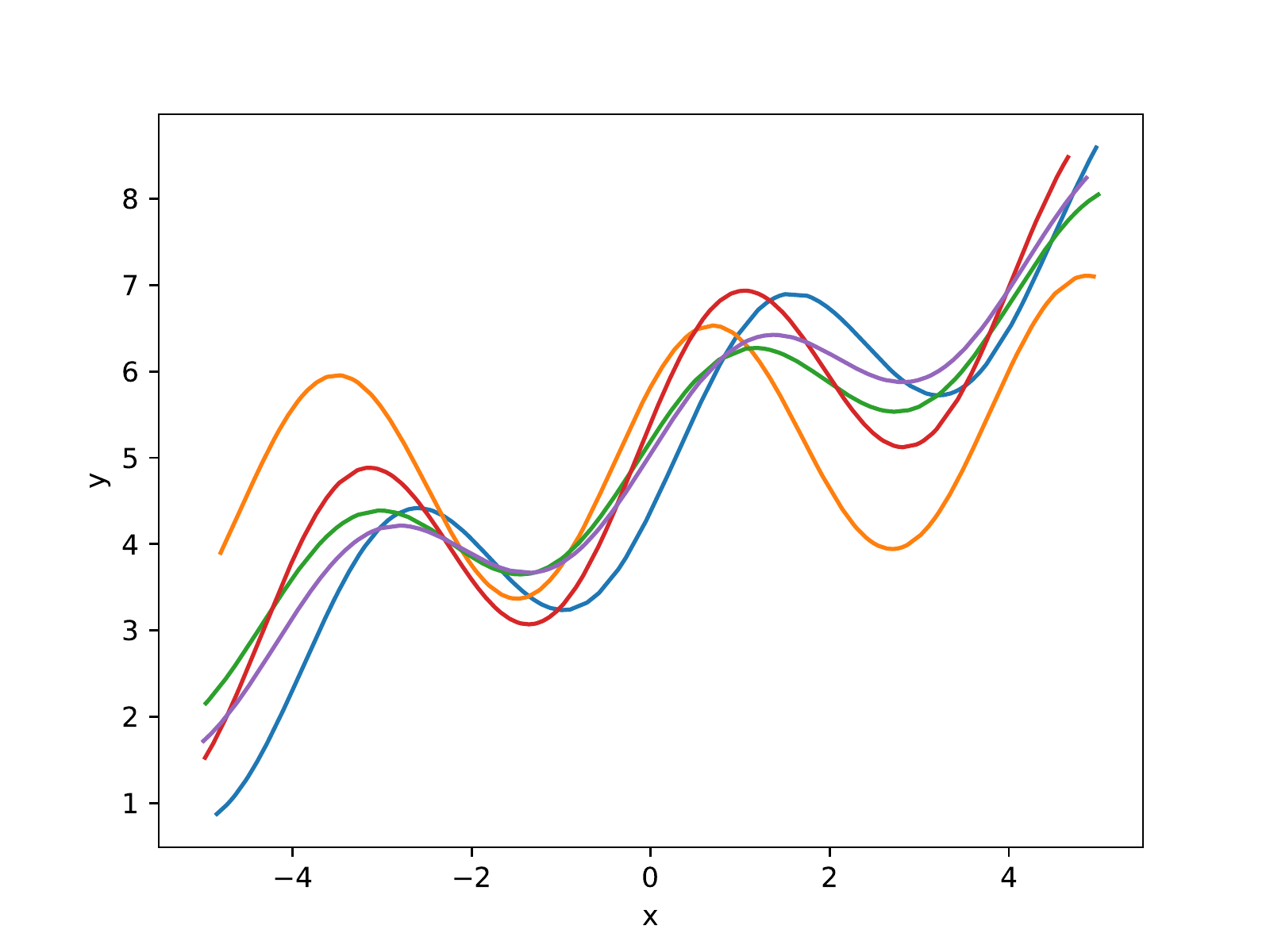}
        \caption{Sinusoid tasks} \label{fig:sin_tasks}
    \end{subfigure}%
    ~
 \begin{subfigure}{0.5\textwidth}
        \centering
        \includegraphics[width=0.95\textwidth]{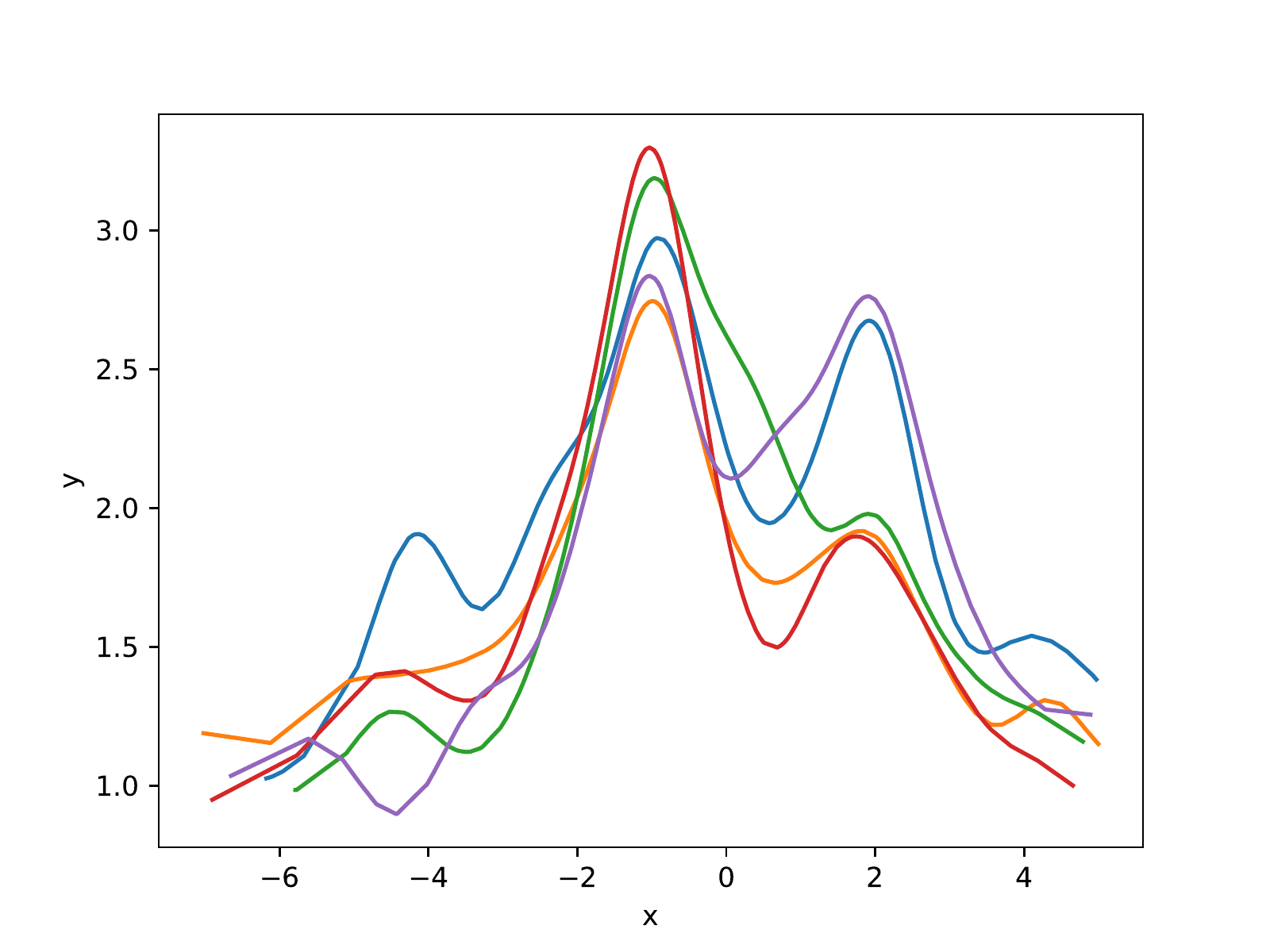}
        \caption{Cauchy tasks} \label{fig:cauchy_tasks}
    \end{subfigure}%
  \caption{Depiction of tasks (i.e., functions) sampled from the Sinusoid and Cauchy task environment, respectively. Note that the Cauchy task environment is two-dimensional ($\text{dim}(\calX) = 2$), while (b) displays a one-dimensional projection.}
\end{figure}
Figure~\ref{fig:sin_tasks} depicts functions $f_{a, b, c, \beta}$ with parameters sampled according to (\ref{eq:param_sampling_sinusoid}). To draw training samples from each task, we draw $x$ uniformly from $\calU(-5, 5)$ and add Gaussian noise with standard deviation $0.1$ to the function values $f(x)$:
\begin{equation}
x \sim \calU(-5, 5)~, \qquad y \sim \calN(f_{a, b, c, \beta} (x), 0.1^2) \;.
\end{equation}

\subsubsection{Cauchy}
Each task of the Cauchy environment can be interpreted as a two dimensional mixture of Cauchy distributions plus a function sampled from a Gaussian process prior with zero mean and SE kernel function $k(x, x') = \exp\left( \frac{||x - x'||_2^2}{2l} \right)$ with $l=0.2$. The (unnormalized) mixture of Cauchy densities is defined as:
\begin{equation}
m(x) = \frac{6}{\pi \cdot (1 + ||x-\mu_1||^2_2)}  + \frac{3}{\pi \cdot (1 + ||x-\mu_2||^2_2)} \;,
\end{equation}
with $\mu_1 = (-1, -1)^\top$ and $\mu_2 = (2, 2)^\top$. 

Functions from the task environments are sampled as follows:
\begin{equation} \label{eq:Cauchy_sampling}
f(x) = m(x) + g(x) ~, \qquad g \sim \mathcal{GP}(0, k(x,x')) \;.
\end{equation}

Figure~\ref{fig:cauchy_tasks} depicts a one-dimensional projection of functions sampled according to (\ref{eq:Cauchy_sampling}). To draw training samples from each task, we draw $x$ from a truncated normal distribution and add Gaussian noise with standard deviation $0.05$ to the function values $f(x)$:  \begin{equation}
x:= \min\{\max\{\tilde{x}, 2\}, -3\} ~, ~~ \tilde{x} \sim \calN(0, 2.5^2) ~, \qquad  y \sim \calN(f(x), 0.05^2) \;.
\end{equation}

\subsubsection{SwissFEL}

Free-electron lasers (FELs) accelerate electrons to very high speed in order to generate shortly pulsed laser beams with wavelengths in the X-ray spectrum. These X-ray pulses can be used to map nanometer scale structures, thus facilitating experiments in molecular biology and material science. The accelerator and the electron beam line of a FEL consist of multiple magnets and other adjustable components, each of which has several parameters that experts adjust to maximize the pulse energy \citep{kirschner2019linebo}. Due do different operational modes, parameter drift, and changing (latent) conditions, the laser's pulse energy function, in response to its parameters, changes across time. As a result, optimizing the laser's parameters is a recurrent task.

\begin{figure}[h]
    \centering
    \includegraphics[width=0.7\linewidth]{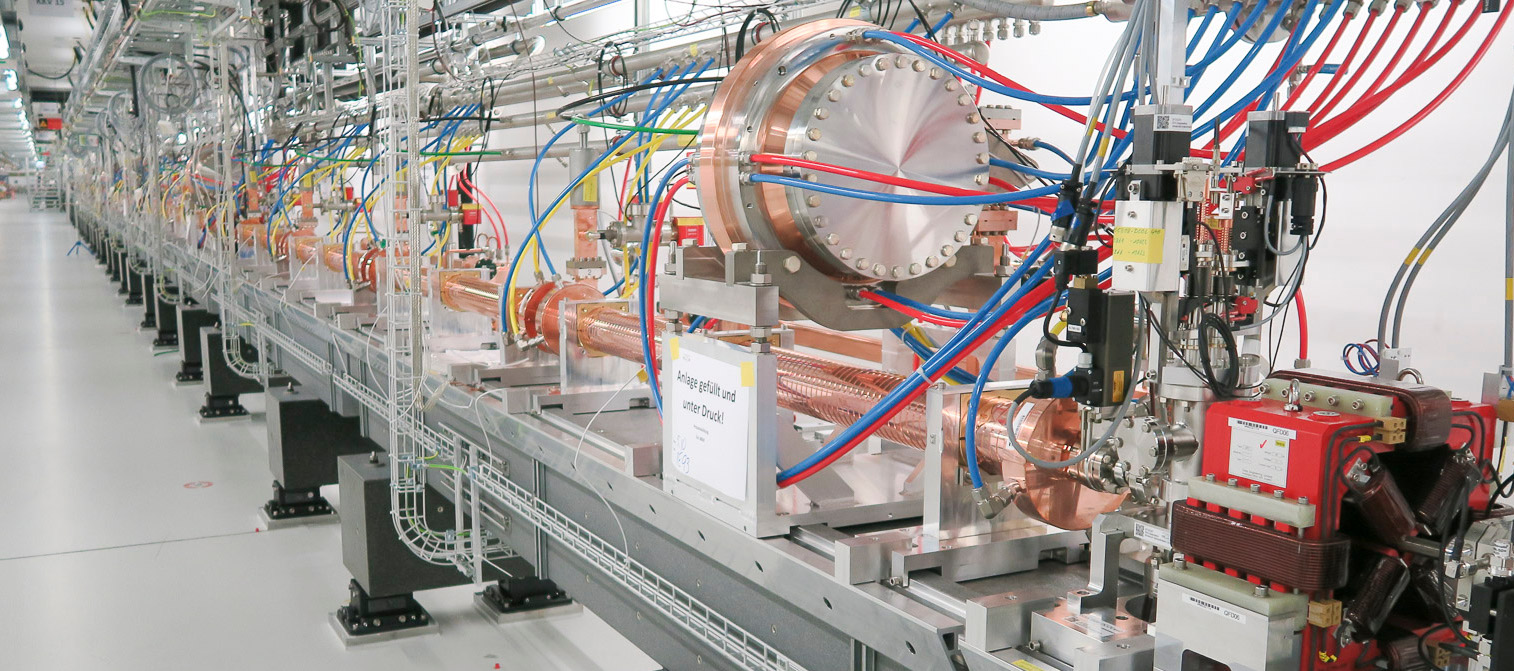}
    \caption{Accelerator of the Swiss Free-Electron Laser (SwissFEL).}
    \label{fig:swissfel}
\end{figure}

Overall, our meta-learning environment consists of different parameter optimization runs (i.e., tasks) on the SwissFEL, an 800 meter long laser located in Switzerland \citep{milne2017swissfel}. A picture of the SwissFEL is shown in Figure~\ref{fig:swissfel}. The input space, corresponding to the laser's parameters, has 12 dimensions whereas the regression target is the pulse energy (1-dimensional). For details on the individual parameters, we refer to \citet{kirschner2019swissfel}. For each run, we have around 2000 data points. Since these data-points are generated with online optimization methods, the data are non-i.i.d.\ and get successively less diverse throughout the optimization. As we are concerned with meta-learning with limited data and want to avoid issues with highly dependent data points, we only take the first 400 data points per run and split them into training and test subsets of size 200. Overall, we have 9 runs (tasks) available. 5 of those runs are used for meta-training and the remaining 4 runs are used for meta-testing. 

\subsubsection{PhysioNet}

The 2012 Physionet competition \citep{silva2012predicting} published an open-access dataset of patient stays on the intensive care unit (ICU).
Each patient stay consists of a time series over 48 hours, where up to 37 clinical variables are measured.
The original task in the competition was binary classification of patient mortality, but due to the large number of missing values (around 80~\% across all features), the dataset is also popular as a test bed for time series prediction methods, especially using Gaussian processes \citep{fortuin2019multivariate}.

In this work, we treat each patient as a separate task and the different clinical variables as different environments.
We use the Glasgow coma scale (GCS) and hematocrit value (HCT) as environments for our study, since they are among the most frequently measured variables in this dataset. From the dataset, we remove all patients where less than four measurements of CGS (and HCT respectively) are available. From the remaining patients we use 100 patients for meta-training and 500 patients each for meta-validation and meta-testing. Here, each patient corresponds to a task. Since the number of available measurements differs across patients, the number of training points $m_i$ ranges between 4 and 24.

\subsubsection{Berkeley-Sensor}

We use data from 46 sensors deployed in different locations at the Intel Research lab in  Berkeley \citep{intel_sensor_data}. The dataset contains 4 days of data, sampled at 10 minute intervals. Each task corresponds to one of the 46 sensors and requires auto-regressive prediction, in particular, predicting the next temperature measurement given the last 10 measurement values. In that, 36 sensors (tasks) with data for the first two days are use for meta-training and whereas the remaining 10 sensors with data for the last two days are employed for meta-testing. Note, that we separate meta-training and -testing data both temporally and spatially since the data is non-i.i.d. For the meta-testing, we use the 3rd day as context data, i.e. for target training and the remaining data for target testing.

\subsection{Experimental Methodology}
\label{appendix:exp_methodology}

In the following, we describe our experimental methodology and provide details on how the empirical results reported in Section~\ref{sec:experiments} were generated. 
Overall, evaluating a meta-learner consists of two phases, {\em meta-training} and {\em meta-testing}, outlined in Appendix \ref{appendix:mll_estimate}. The latter can be further sub-divided into {\em target training} and {\em target testing}. Figure~\ref{fig:overview} illustrates these different stages for our PAC-Bayesian meta-learning framework.



The outcome of the training procedure is an approximation for the generalized Bayesian posterior $Q^*(S,P)$ (see Appendix \label{ref:mll_estimate}), pertaining to an unseen task $\tau = (\calD, m) \sim \calT$ from which we observe a dataset $\tilde{S} \sim \calD$. In {\em target-testing}, we evaluate its predictions on a held-out test dataset $\tilde{S}^* \sim \calD$ from the same task. For PACOH-NN, NPs and MLAP the respective predictor outputs a probability distribution $\hat{p}(y^*|x^*, \tilde{S})$ for the $x^*$ in $\tilde{S}^*$. The respective mean prediction corresponds to the expectation of $\hat{p}$, that is $\hat{y} = \hat{\E}(y^*|x^*, \tilde{S})$. In the case of MAML, only a mean prediction is available. Based on the mean predictions, we compute the {\em root mean-squared error (RMSE)}:
\begin{equation}
    \text{RMSE} =  \sqrt{ \frac{1}{|\tilde{S}^*|} \sum_{(x^*, y^*) \in S^*} (y^* - \hat{y})^2} \;.
\end{equation}
and the {\em calibration error} (see Appendix \ref{appendix:calib_error}). Note that unlike e.g. \citet{rothfuss2019noise} who report the test log-likelihood, we aim to measure the quality of mean predictions and the quality of uncertainty estimate separately, thus reporting both RMSE and calibration error. 

The described meta-training and meta-testing procedure is repeated for five random seeds that influence both the initialization and gradient-estimates of the concerned algorithms. The reported averages and standard deviations are based on the results obtained for different seeds.

\subsubsection{Calibration Error} \label{appendix:calib_error}
The concept of calibration applies to probabilistic predictors that, given a new target input $x_i$, produce a probability distribution $\hat{p}(y_i|x_i)$ over predicted target values $y_i$. 

\textbf{Calibration error for regression.}
Corresponding to the predictive density, we denote a predictor's cumulative density function (CDF) as $\hat{F}(y_j|x_j) = \int_{-\infty}^{y_j} \hat{p}(y|x_i) dy$. For confidence levels $0 \leq q_h < ... < q_H \leq 1$, we can compute the corresponding empirical frequency
\begin{equation}
\hat{q}_h = \frac{ | \{ y_j ~ | ~ \hat{F}(y_j|x_j) \leq q_h, j=1, ..., m \} | }{m} \;,
\end{equation} 
based on dataset $S=\{(x_i, y_i)\}_{i=1}^m$ of $m$ samples. If we have calibrated predictions we would expect that $\hat{q}_h \rightarrow q_h$ as $m \rightarrow \infty$. Similar to \citep{Kuleshov2018}, we can define the calibration error as a function of residuals $\hat{q}_h - q_h$, in particular,
\begin{equation} \label{eq:calib_err}
\text{calib-err} = \frac{1}{H} \sum_{h=1}^H |\hat{q}_h - q_h| \;.
\end{equation}
Note that we while \citep{Kuleshov2018} reports the average of squared residuals $|\hat{q}_h - q_h|^2$, we report the average of absolute residuals $|\hat{q}_h - q_h|$ in order to preserve the units and keep the calibration error easier to interpret. In our experiments, we compute (\ref{eq:calib_err}) with $M=20$ equally spaced confidence levels between 0 and 1.

\textbf{Calibration error for classification.} Our classifiers output a categorical probability distribution $\hat{p}(y=k|x)$ for $k=1,..., C$ where $\calY = \{1, ..., C\}$ with $C$ denoting the number of classes. The prediction of the classifier is the most probable class label, i.e., $\hat{y_j} = \argmax_{k} \hat{p}(y_j=k|x_j)$. Correspondingly, we denote the classifiers confidence in the prediction for the input $x_j$ as $\hat{p}_j := \hat{p}(y_j=\hat{y_j}|x_j)$. Following, the calibraton error definition of \citet{guo2017calibration}, we group the predictions into $H=20$ interval bins of size $1/H$ depending on their prediction confidence. In particular, let $B_h = \{j ~|~ p_j \in \left(\frac{h-1}{H}, \frac{h}{H} \right]\} $ be the set of indices of test points $\{(x_j,y_j)\}_{j=1}^m$ whose prediction fall into the interval $\left(\frac{h-1}{H}, \frac{h}{H} \right] \subseteq (0,1]$. Formally, we define the accuracy of within a bin $B_h$ as 
\begin{equation}
    \text{acc}(B_h) = \frac{1}{|B_h|} \sum_{j\in B_h} \mathbf{1}(\hat{y_i} = y_j)
\end{equation}
and the average confidence within a bin as
\begin{equation}
    \text{conf}(B_h) = \frac{1}{|B_h|} \sum_{j\in B_h} \hat{p}_j \;.
\end{equation}
If the classifier is calibrated, we expect that the confidence of the classifier reflects it's accuracy on unseen test data, that is, $\text{acc}(B_h)=\text{conf}(B_h) ~ \forall h = 1,...., H$. As proposed of \citet{guo2017calibration}, we use the expected calibration error (ECE) to quantify how much the classifier deviates from this criterion: More precisely, in Table \ref{table:cal_err_reg}, we report the ECE with the following definition:
\begin{equation}
    \text{calib-err} = \text{ECE} = \sum_{h=1}^H \frac{|B_h|}{m} \big|\text{acc}(B_h) - \text{conf}(B_h) \big|
\end{equation}
with $m$ denoting the overall number of test points.
\subsection{Hyper-Parameter Selection}
For each of the meta-environments and algorithms, we ran a separate hyper-parameter search to select the hyper-parameters. In particular, we use the \texttt{hyperopt}\footnote{http://hyperopt.github.io/hyperopt/} package \citep{bergstra13} which performs Bayesian optimization based on regression trees. As optimization metric, we employ the average log-likelihood, evaluated on a separate validation set of tasks.

The scripts for reproducing the hyper-parameter search for PACOH-GP are included in our code repository\footnote{\href{https://github.com/jonasrothfuss/meta_learning_pacoh}{\texttt{https://github.com/jonasrothfuss/meta\_learning\_pacoh}}}. 
For the reported results, we provide the selected hyper-parameters and detailed evaluation results under 
\href{https://tinyurl.com/s48p76x}{\texttt{https://tinyurl.com/s48p76x}}.

\subsection{Meta-Learning for Bandits - Vaccine Development}

In this section, we provide additional details on the experiment in Section \ref{sec:bandit}.

We use data from \citet{widmer2010inferring} which contains the binding affinities ($\text{IC}_{50}$ values) of many peptide candidates to seven different MHC-I alleles. Peptides with $\text{IC}_{50} > 500$nM are considered non-binders, all others binders. Following \citet{krause2011contextual}, we convert the $\text{IC}_{50}$ values into negative log-scale and normalize them such that 500nM corresponds to zero, i.e. $r := -\log_{10}(\text{IC}_{50}) + \log_{10}(500)$ with is used as reward signal of our bandit.

\begin{table}[t]
\centering
\begin{tabular}{l|ccccc}
 Allele & A-0202 & A-0203 & A-0201 &  A-2301 &  A-2402 \\ \hline
 $m_i$ & 1446 & 1442 & 3088 & 103  & 196 \\ 
\end{tabular}
\caption{MHC-I alleles used for meta-training and their corresponding number of meta-training samples $m_i$.}
\label{tab:alleles}
\end{table}

We use 5 alleles to meta-learn a BNN prior. The alleles and the corresponding number of data points, available for meta-training, are listed in Table \ref{tab:alleles}. The most genetically dissimilar allele (A-6901) is used for our bandit task. In each iteration, the experimenter (i.e. bandit algorithm) chooses to test one peptide among the pool of 813 candidates and receives $r$ as a reward feedback. Hence, we are concerned with a 813-arm bandit wherein the action $a_t \ \in \{1, ..., 813\} = \calA$ in iteration $t$ corresponds to testing $a_t$-th peptide candidate. In response, the algorithm receives the respective negative log-$\text{IC}_{50}$ as reward $r(a_t)$.

As metrics, we report the \emph{average regret}
$$
R^{avg.}_T :=  \max_{a \in \calA} r(a) - \frac{1}{T} \sum_{t=1}^T r(a_t)
$$
and the \emph{simple regret}
$$
R^{simple}_T :=  \max_{a \in \calA} r(a) - \max_{t=1,...,T} r(a_t)
$$

To ensure a fair comparison, the prior parameters of the GP for GP-UCB and GP-TS are meta-learned by minimizing the GP's marginal log-likelihood on the five meta-training tasks. For the prior, we use a constant mean function and tried various kernel functions (linear, SE, Matern). Due to the 45-dimensional feature space, we found the linear kernel to work the best. So overall, the constant mean and the variance parameter of the linear kernel are meta-learned.

\subsection{Further Experimental Results} \label{appendix:exp-results}

\paragraph{Meta-overfitting}
In order to investigate whether the phenomenon of meta-overfitting, which we have observed consistently for PACOH-MAP and MLL, is also relevant to other meta-learning methods (MAML and NPs), we also report the meta-train test error and the meta-test test error across different numbers of tasks. The results, analogous to Figure \ref{fig:meta_overfitting}, are plotted in Figure \ref{fig:meta_overfitting_maml_nps}, showing a significant difference between the meta-train and meta-test error that vanishes as the number of tasks becomes larger. Once more, this supports our claim that meta-overfitting is a relevant issue and should be addressed in a principled manner.  

\begin{figure*}[ht]
\centering
\includegraphics[width=0.9\textwidth]{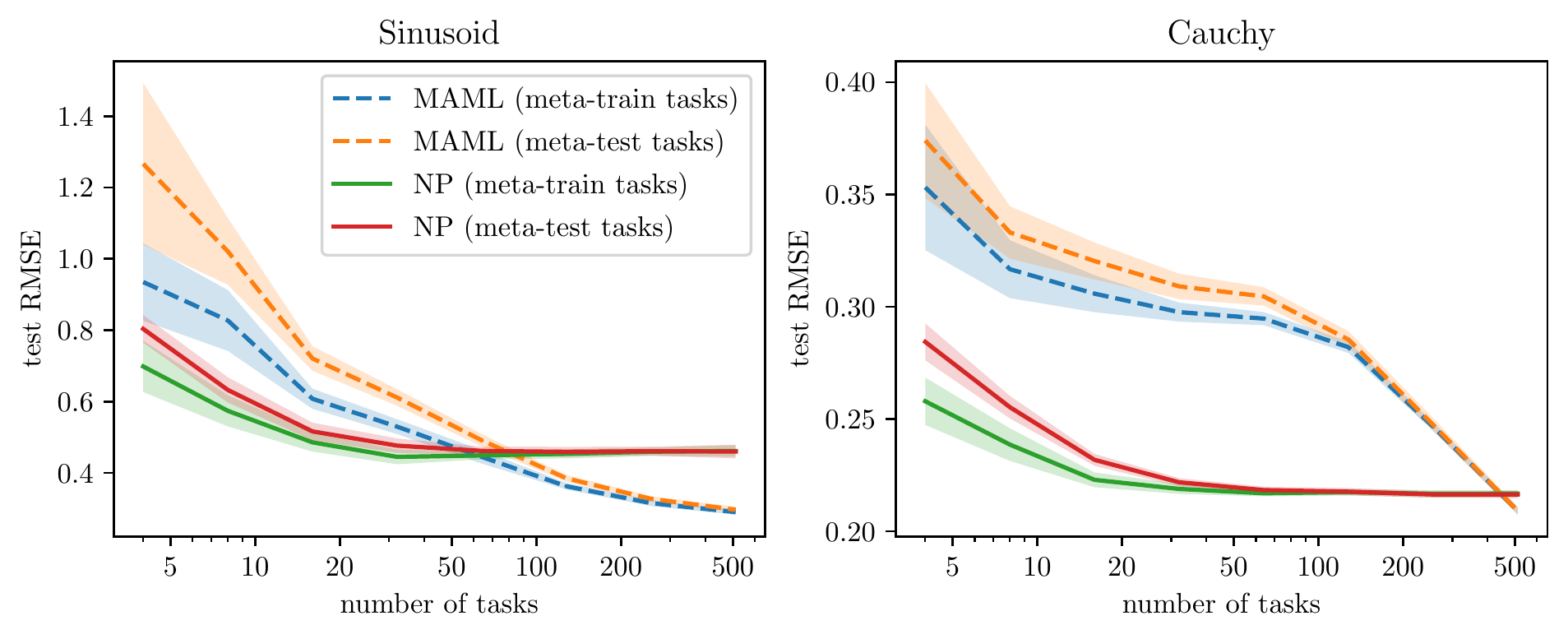}
\caption{Test RMSE measured on the meta-training tasks and the meta-testing tasks as a function of the number of meta-training tasks for MAML and NPs. The performance gap between the meta-train and meta-test tasks clearly demonstrates overfitting on the meta-level for both methods.} 
\label{fig:meta_overfitting_maml_nps}
\end{figure*}



\end{document}